\documentclass{article}

\usepackage[main, final]{neurips_2026}

\usepackage[utf8]{inputenc} 
\usepackage[T1]{fontenc}    
\usepackage{hyperref}       
\usepackage{graphicx}
\usepackage{url}            
\usepackage{booktabs}       
\usepackage{amsfonts}       
\usepackage{nicefrac}       
\usepackage{microtype}      
\usepackage{xcolor}         
\usepackage{amsmath,amssymb,amsthm}
\usepackage{xcolor}
\usepackage{tcolorbox}
\usepackage{subcaption}
\usepackage{thmtools}
\usepackage{thm-restate}

\usepackage{bm}

\PassOptionsToPackage{capitalize,nameinlink}{cleveref}
\usepackage{cleveref}

\newtheorem{theorem}{Theorem}[section]
\newtheorem{lemma}[theorem]{Lemma}

\newtheorem{remark}{Remark}[section]

\newcommand{\R}{\mathbb{R}}

\newcommand{\vertiii}[1]{{\left\vert\kern-0.25ex\left\vert\kern-0.25ex\left\vert #1 
    \right\vert\kern-0.25ex\right\vert\kern-0.25ex\right\vert}}

\newcommand{\beq}{\begin{eqnarray*}}
\newcommand{\eeq}{\end{eqnarray*}}
\newcommand{\beqn}{\begin{eqnarray}}
\newcommand{\eeqn}{\end{eqnarray}}

\usepackage{ifmtarg}
\usepackage{xifthen}%
\newcommand{\ent}[1][]{%
\ifthenelse{\isempty{#1}}{%
\mathrm{H}
}{
\mathrm{H}^{(#1)}
}}

\newcommand{\loch}[1][]{%
\ifthenelse{\isempty{#1}}{%
\mathrm{h}
}{
\mathrm{h}^{(#1)}
}}

\newcommand{\bb}{\mathbf{b}}
\newcommand{\bu}{\mathbf{u}}
\newcommand{\bv}{\mathbf{v}}

\newcommand{\inner}[1]{\left\langle#1\right\rangle}

\DeclareUnicodeCharacter{221E}{\ensuremath{\infty}}

\newcommand{\norm}[1]{\left\lVert#1\right\rVert}

\newcommand{\naturals}{\mathbb{N}}

\newcommand{\pr}{\text{Pr}}

\title{Geometric Factual Recall in Transformers}

\author{
Shauli Ravfogel\textsuperscript{*}$^{1}$\quad 
Gilad Yehudai\textsuperscript{*}$^{2}$\quad
Joan Bruna$^{1}$\quad
Alberto Bietti$^{3}$\\
$^{1}$New York University\quad 
$^{2}$ Technion, Israel\quad
$^{3}$Flatiron Institute
}

\begin{document}
\maketitle

\footnotetext[1]{\textsuperscript{*}Equal contribution.}

\begin{abstract}
How do transformer language models memorize factual associations? A common account views internal weight matrices as an associative memory. That account requires parameter counts that scale linearly with the number of facts. We develop a theoretical and empirical account of an alternative, \emph{geometric} form of memorization in which learned embeddings encode relational structure directly, and the MLP plays a qualitatively different role. In a controlled setting where a single-layer transformer memorizes random bijections from subjects to a shared attribute set, we prove that a logarithmic embedding dimension is sufficient. Subject embeddings encode \emph{linear superpositions} of their associated attribute vectors, and a small MLP acts as a relation-conditioned selector that extracts the relevant attribute from the embedding. We extend these results to the multi-hop setting---chains of relational queries such as ``Who is the mother of the wife of $x$?''---and provide constructions with and without chain-of-thought that exhibit a provable capacity--depth tradeoff, complemented by a matching information-theoretic lower bound. Empirically, gradient descent discovers solutions with the predicted structure. Causal experiments support the proposed mechanism: when subject embeddings are appropriately re-initialized, the frozen MLP transfers zero-shot to entirely new bijections, showing it has learned a generic selection mechanism rather than functioning as a key-value memory.
\end{abstract}

\section{Introduction}

A defining capability of large language models (LLMs) is their ability to store vast collections of facts and retrieve them on demand \citep{petroni2019language,roberts2020much}. Facts expressed in natural language have rich structure: the same entity participates in many relations (e.g., \textit{birth place}, \textit{occupation}), and models are routinely asked to \emph{compose} facts—for example, answering questions about the mother of the author of a given book. This raises two fundamental questions: \emph{how do transformer language models memorize factual associations, and what latent representations support factual recall?}

Existing theoretical accounts emphasize an \emph{algebraic} form of memorization in which MLPs or attention matrices store pairs of input-output embeddings which are themselves unstructured, e.g., random nearly-orthogonal vectors \citep{bietti2023birth,nichani2024understanding}. Recent empirical work suggests a different mechanism: transformers with \emph{learned} embeddings memorize in a \emph{geometric} manner, developing representations whose structure encodes the topology of the stored relational information to be recalled \citep{noroozizadeh2025deep}. This leads to learned embeddings that are structured, with related entities leading to correlated embeddings. Phenomena such as the linearity of relation encoding \citep{hernandez2023linearity, merullolinear} further hint at structured, low-dimensional geometric solutions.
The goal of this work is to develop a theoretical understanding of this geometric form of memorization, and to verify that the predicted structures emerge in practice.

We focus on the multi-relation setting, in which a set of $N$ \emph{subjects} 
(e.g., people's names) are each associated with $R$ \emph{factual 
relations}—such as birth place, occupation, or native language—and each 
(subject, relation) pair maps to a single \emph{attribute} (e.g., London, 
physicist, English). A transformer LM is trained to predict the next token 
in sequences expressing these relations, or compositions of them. At a high level, our results show that learnable embeddings allow a 
transformer to memorize $N$ subjects across $R$ relations using an embedding 
dimension and a number of attention/MLP parameters that is only \emph{logarithmic} in $N$, by encoding relational 
structure directly into the geometry of the embedding space. The mechanism 
is qualitatively different from associative-memory accounts based on random, 
near-orthogonal embeddings: subjects are encoded as \emph{superpositions} of 
their $R$ associated attribute vectors, and the MLP acts as a 
\emph{relation-conditioned selector} that extracts the queried attribute via 
ReLU gating. We further show that this geometric mechanism extends to 
multi-hop relational queries—\emph{``Who is the mother of the wife of $x$?''}—
and that chain-of-thought (CoT) provably circumvents an otherwise unavoidable capacity bottleneck. 

Does gradient descent discover solutions with the same structure? We address this question through extensive experiments on synthetic data in the shared-attribute setting (\cref{sec:experiments}). Across a variety of embedding dimensions and numbers of relations, we find that: (i)~the capacity threshold $d = \Theta(R \log N)$ predicted by theory is validated empirically; (ii)~it is possible to approximate learned subject embeddings as linear superpositions of per-relation attribute vectors; (iii)~ causal intervention experiments confirm that the MLP implements a relation-specific selector; and (iv)~the MLP transfers to entirely \emph{new} sets of facts (bijections) without retraining, demonstrating that it has learned a generic selection mechanism rather than memorizing specific facts. These findings support the view that factual recall in transformers can be implemented via structured geometric representations rather than brute-force memorization.
We note that while we focus on encoding the geometric structure in the embeddings themselves, in real models this structure might emerge at intermediate layers, particularly for entities defined over multiple tokens. Such encoding can still be helpful for efficient knowledge manipulation at later layers (i.e., only a simple selection mechanism is needed on top). We perform preliminary experiments on pretrained LLMs in \cref{sec:real-lm-probing}.

\paragraph{Our contributions:}
\begin{itemize}
    \item \textbf{Logarithmic memorization with shared attributes 
    (\cref{thm:shared attr}).} A single-layer transformer with a small 
    MLP memorizes all $NR$ facts in dimension $d = O(R \log N)$, the regime 
    of practical interest being $R \ll N$. The construction encodes each 
    subject as a linear superposition of its $R$ attribute vectors, decoded 
    by a relation-conditioned ReLU gate in the MLP.
    

    \item \textbf{Capacity--depth tradeoff for multi-hop recall 
    (\cref{thm:k-hop-non-cot,thm:k-hop-cot}).} Without CoT, a transformer 
    solving $k$-hop queries must pay either in MLP width ($\Omega(NR)$) or in 
    embedding dimension ($\tilde \Omega(R^k)$). With CoT, a single-layer 
    transformer with $d = O(R\log N + k)$ suffices for any $k$. A counting argument over $R$-regular 
    edge-colored graphs (\cref{thm:k-hop lower bound}) shows that the no-CoT tradeoff above is 
    near-optimal.

    \item \textbf{Empirical verification (\cref{sec:experiments}).} The 
    predicted threshold $d = \Theta(R \log N)$ emerges under gradient descent. Linear readouts, causal 
    interventions, and a freeze-and-swap transfer experiment confirm both 
    the superposition structure of subject embeddings and the generic, 
    relation-conditioned selector role of the MLP.
\end{itemize}

\section{Preliminaries and Problem Setup}\label{sec:setup}
We start by describing the toy setting for modeling multi-hop factual recall.

\subsection{Task Formulation: Multi-Hop Factual Recall}
For $N\in \mathbb{N}$ we define $[N] = \{1,\dots,N\}$.
Let $[N]$ denote a set of subjects and $[R]$ a set of relations, for $N, R \in \mathbb{N}$. We associate each relation $r \in [R]$ with a ground-truth bijection $g_r: [N] \to [N]$. For a given sequence depth $k \ge 1$, we train an autoregressive language model over the vocabulary $[N] \cup [R]$ on sequences of the form $x = (s_0, r_1, r_2, \dots, r_k, y) \in [N] \times [R]^k \times [N]~$,
where the final token is determined by the composition $y = (g_{r_k} \circ \cdots \circ g_{r_1})(s_0) \in [N]~$.
The model parameterizes a next-token distribution $p_\theta(x_{t+1} \mid x_{\le t})$ and is trained with the standard causal language modeling objective $\mathcal{L}(\theta) = -\mathbb{E}_{x}\!\left[\sum_{t=0}^{k} \log p_\theta(x_{t+1} \mid x_{\le t})\right]~$.
At inference, we evaluate the model's prediction of the final token conditioned on the preceding prefix, $\hat{y} = \arg\max_{v \in [N]} p_\theta(v \mid s_0, r_1, \dots, r_k)$. 
The single-hop regime ($k=1$), where sequences take the form $(s_0, r_1, y)$, constitutes a strictly harder generalization of the factual recall task formalized by \cite{nichani2024understanding}. Indeed, their theoretical framework assumes mutually disjoint attribute codomains (i.e., $g_r: [N] \to Y_r$ where $Y_i \cap Y_j = \emptyset$ for $i \neq j$). In contrast, our definition enforces a shared codomain $[N]$, inducing a non-trivial routing bottleneck due to output collisions.
For completeness, \Cref{appen:disjoint_attr} details a learned-embedding construction for the disjoint-attribute setting. By explicitly structuring the representations, this construction achieves a significant reduction in the requisite MLP parameter capacity compared to the bounds in \cite{nichani2024understanding}, dropping from $\mathcal{O}(N)$ to $\mathcal{O}(\log N)$. The remainder of this work considers only the generalized, shared-attribute setting.

\subsection{Architecture}

We consider a transformer architecture taking as input a sequence of $n$ tokens with token embeddings $x_1, \ldots, x_n \in \mathbb{R}^d$. We denote by $X^{(0)} \in \mathbb{R}^{d \times n}$ the matrix where the $i$-th column is the embedding of the $i$-th token.
Each layer $\ell \in [L]$ applies a multi-head self-attention mechanism followed by a position-wise feed-forward network (MLP). We denote the input to layer $\ell$ by $X^{(\ell-1)}$. The self-attention mechanism at layer $\ell$ with $H$ heads is parameterized by matrices $W_Q^{(h, \ell)}, W_K^{(h, \ell)} \in \mathbb{R}^{d \times d}$ and $W_V^{(h, \ell)} \in \mathbb{R}^{d \times d}$ for each head $h \in [H]$. 
We first define the attention scores matrix $A^{(h, \ell)} \in \mathbb{R}^{n \times n}$ for head $h$ as:
\begin{equation}
    A^{(h, \ell)} = \text{softmax}\left( (W_K^{(h, \ell)} X^{(\ell-1)})^\top (W_Q^{(h, \ell)} X^{(\ell-1)}) \right)
\end{equation}
The output of the multi-head attention mechanism, $Z^{(\ell)} \in \mathbb{R}^{d \times n}$, is then computed by aggregating the value heads:
\begin{equation}
    Z^{(\ell)} = \sum_{h=1}^H W_{V}^{(h, \ell)} X^{(\ell-1)} (A^{(h, \ell)})^\top
\end{equation}
This is followed by a residual connection, such that the intermediate representation is $\tilde{X}^{(\ell)} = Z^{(\ell)} + X^{(\ell-1)}$.

Finally, we apply an MLP, denoted by $\text{MLP}^{(\ell)}: \mathbb{R}^d \to \mathbb{R}^d$, which operates on each token (column) independently and has a hidden dimension of $m$, which may differ from $d$. We call $m$, the number of hidden neurons in each hidden layer, the width of the MLP. The MLP consists of two linear transformations with a ReLU activation. In the theoretical constructions, we sometimes use a $2$-hidden layer MLP instead of $1$, and state this explicitly. The output of the layer is $X^{(\ell)} = \tilde{X}^{(\ell)} + \text{MLP}^{(\ell)}(\tilde{X}^{(\ell)})$.
Our theoretical analysis does not explicitly include normalization layers, although this does not limit the generality of our results, as it is possible to implement degenerate normalization layers that act as the identity function (see, e.g., \citet{sanford2024representational}). In contrast, our empirical evaluations utilize standard transformer layers that include normalization.

\section{Related Work}\label{sec:related}

We overview core related work and defer additional references to \cref{app:extended-related}.

A line of work on factual recall localizes factual memorization in pretrained transformers to specific components, reading FFN layers as key--value memories \citep{geva2021transformer}, identifying mid-layer MLPs as the dominant store of subject--attribute associations via causal interventions \citep{meng2022locating}, and tracing recall as attention transporting the subject and MLPs performing the lookup \citep{geva2023dissecting}. A parallel editing literature makes practical use of this view through targeted weight updates \citep{dai2022knowledge,meng2022locating,meng2022memit}. 

\citet{bietti2023birth} introduce an associative-memory view in which transformer weights are outer products over near-orthogonal embedding pairs, with roots in the neural computation literature~\citep{hopfield1982neural,krotov2016dense,ramsauer2021hopfield}. \citet{cabannesscaling} further studies the capacity of such associative memory matrices in a simple setting, and illustrates benefits of learned embeddings, though they only consider basic pairwise associations with no relations. Closest to our setting, \citet{nichani2024understanding} prove that a one-layer transformer with \emph{fixed random spherical} embeddings stores all $NR$ subject--relation--attribute associations using $\tilde\Omega(NR)$ attention or MLP parameters. Two questions are left open: whether \emph{learnable} embeddings can reduce $d$ to be merely logarithmic, without inflating the MLP, and what the resulting solution looks like. \cref{thm:disjoint-attr,thm:shared attr} answer both: learnable embeddings reduce $d$ to logarithmic, with a structure (superposition rather than near-orthogonal random codes) qualitatively different from the random-embedding regime. 

A long line of work establishes worst-case memorization bounds for neural networks \citep{vardi2022optimal,egosi2025logarithmic} and transformers \citep{yun2020transformers,kim2023provable,kajitsuka2024transformers}, but these depend on the number of examples and sequence length and do not exploit relational structure. Closest to us, \citet{dugan2025constructing} give an encoder--decoder MLP construction that stores a fact set $f:[K]\to[V]$ in $\Theta(K \log V)$ parameters, matching the information-theoretic lower bound for well-conditioned value embeddings. Their setting is single-relation: keys and values are given, and the construction produces MLP weights that map between them. Our setup is different in two ways. First, we fix a minimal MLP and instead \emph{construct the subject embeddings}, encoding each subject as a superposition of its $R$ attribute codes that the MLP decodes via ReLU gating. Second, the multi-relation structure is central to our analysis: treating our $NR$ subject--relation pairs as a single fact set in their framework would cost $\Theta(NR \log N)$ parameters.

Finally, \cref{thm:k-hop-non-cot} and \cref{thm:k-hop-cot} exhibit a capacity--depth tradeoff: with CoT, $d = O(R \log N + k)$ suffices for any number of hops; without CoT, either $d$ or MLP width must blow up exponentially or linearly in $k$. This parallels circuit-complexity work showing that bounded-depth transformers cannot solve certain tasks without super-polynomial size but constant-size autoregressive transformers can with CoT \citep{feng2023revealing,merrill2024expressive,li2024chain}. Our setting is narrower (multi-hop relational queries) but quantitatively sharper, with explicit constructions in $k$ and $R$ plus a nearly matching information-theoretic lower bound on the no-CoT parameters. In practice, prior work has shown that LLMs struggle with multi-hop questions and that, when they succeed, they may rely on both latent compositional mechanisms and more direct input–output mappings, without an explicit representation of the intermediate answer \citep{yang2024large, biran2024hopping, khandelwal2025language}. These latent compositional mechanisms are studied as computations unfolding across model layers in a single forward pass, rather than as explicit chain-of-thought reasoning. Recently, \citet{gekhman2026thinking} have empirically observed that CoT helps factual recall. Our construction provides one theoretical explanation for the phenomenon.

\section{Factual Recall with Learned Representations}\label{sec:fact_recall_theory}

In this section, we analyze the capacity of transformers to solve factual recall problems when embeddings are treated as learnable parameters. We demonstrate that this flexibility permits embeddings to be structured geometrically, allowing cheap knowledge manipulation with small non-linear models on top of these, instead of expensive lookup based on large MLPs with random embeddings, which is common in previous works \citep{nichani2024understanding, dugan2025constructing}. We begin with a warm-up of the single-hop case, establishing that a logarithmic embedding dimension suffices via a selection mechanism in a ``small'' MLP. We then transition to the $k$-hop setting. In the standard regime of direct prediction, we prove an information-theoretic lower bound on the parameter-dimension trade-off, accompanied by explicit non-CoT constructions demonstrating that this bound is nearly tight. To understand how modern language models overcome such limitations for extended reasoning traces, we naturally introduce Chain-of-Thought (CoT) generation into our setting. Surprisingly, we show that allowing the model to emit intermediate relational steps completely breaks this theoretical bottleneck. CoT bypasses the rigid lower bounds of the direct $k$-hop setting, gracefully reducing the multi-hop capacity requirements back to the complexity of the single-hop case.

\subsection{Warm-Up: Single-Hop Factual Recall}

To build intuition, we first show that a logarithmic embedding dimension suffices for the standard $1$-hop task by using the MLP as a generic selection mechanism rather than a lookup table. 

\begin{theorem}\label{thm:shared attr}
    There is a $1$-layer transformer with a $3$-layer MLP of width~$O(R\log N)$ and embedding dimension $d=6R\log(N) + 1$ that correctly solves the single-hop factual recall problem.
    This can also be achieved without an MLP, by using multi-head attention with~$R$ attention heads, with head dimension~$d_h = 6 \log(N)$ and slightly larger embedding dimension~$d = 6R \log(N) + 6 \log(R) + 1$.
\end{theorem}

The full proof can be found in \Cref{appen:proof_shated_attr}. In the construction, rather than encoding facts in the attention or MLP weights, we stack the $R$ target attributes directly into the subject embedding. Specifically, each subject embedding takes the form of a block vector containing the representations of its $R$ possible answers, one for each relation.  The 3-layer MLP then acts as a relation-conditioned selector, extracting only the block corresponding to the relation from the query. This construction demonstrates that even in the shared-attribute regime, a much weaker assumption than in  \citet{nichani2024understanding}, an embedding dimension $d = O(R\log N)$, logarithmic in $N$, and MLP width $O(R\log N)$ remains sufficient for factual recall of $N$ facts. While the total parameter count is similar, this construction uses the already-existing lookup embeddings  instead of computational weights. This reliance on lookup parameters can increase the computational efficiency of the model, something which has recently motivated new architectures with embeddings at each layer \citep{cheng2026conditional}.

\subsection{The Information Bottleneck for multi-hop factual recall}

Generalizing $1$-hop factual recall to the $k$-hop regime exposes a fundamental information-theoretic bottleneck, inducing a strict capacity trade-off between the transformer's parameter count and its embedding dimension. This is highlighted in the following theorem:

\begin{theorem}\label{thm:k-hop lower bound}
Let $R \ge 2$ and $k \ge 2$. Let $p = \lfloor \frac{N}{2 R^{2k}} \rfloor$. If a transformer $\mathcal{T}$ solves the $k$-hop task with zero error without Chain-of-Thought, its internal weights $W$ must satisfy:
\[
W \ge \max \Big\{ \mathcal{B}_{global}, \; \mathcal{B}_{local} \Big\} - R \cdot D 
\]
where the global and local capacity constraints are defined as:
\begin{align*}
\mathcal{B}_{global} &= (R-1)N \log_2\left(\frac{N}{e}\right) - N \cdot D ~,~
\mathcal{B}_{local} = p \Big[ R^k(\log_2 N - 1) - D \Big]~.
\end{align*}
Here $D$ is the number of bits in the embedding of a single token, namely the embedding dimension times the number of bits of precision per coordinate, and $W$ is the number of bits in the self-attention and MLP parameters of $\mathcal{T}$, namely the number of these parameters times their precision. 
\end{theorem}

The lower bound shows a trade-off between $W$, the number of bits in the transformer attention and MLP layers, and $D$, the number of bits in a single token embedding. This trade-off can be seen in three different regimes depending on the size of $D$:

\begin{enumerate}
    \item If $D < R$, then $\mathcal{B}_{global}$ dominates, yielding $W \ge \Omega(N \log N)$. The model must store $NR$ facts inside its weights, yielding a potentially very large transformer (e.g., a large key-value MLP memorizer). 
    \item If $D > R^k$, then both constraints $\mathcal{B}_{global}$ and $\mathcal{B}_{local}$ vanish. The embedding space is large enough to encode the entire $k$-hop evaluation tree of each subject, requiring minimal active computation from $W$. In this solution, the MLP can act as a selector, navigating the $k$-hop evaluation tree, but the embedding dimension is exponential in $k$.
    \item If $R < D < R^k$, then $\mathcal{B}_{local}$ dominates, and we have $\mathcal{B}_{local} \approx \frac{N}{R^k} - D$, bounding $W \gtrsim \frac{N}{R^k}$. 
\end{enumerate}

The formal proof is deferred to \Cref{appen:proof_k_hop_lower_bound}. The lower bound proceeds via a counting argument over the space of $R$-regular directed edge-colored graphs on $N$ vertices. This graph encodes the complete transition system, and locally, the $k$-hop neighborhood of any subject unfolds into a complete $R$-ary tree of depth $k$, enumerating all valid relational paths. $\mathcal{B}_{global}$ bounds the global state space: the transformer's parameters must distinguish between all valid $k$-hop transition matrices. By factoring out constant right-translations of the graph, we apply the Pigeonhole Principle to the remaining $(N!)^{R-1}$ distinct evaluations. $\mathcal{B}_{local}$ bounds the local routing capacity: by greedily extracting $p$ vertex-disjoint $R$-ary subtrees of depth $k$, we isolate paths that the transformer must correctly route without interference, forcing a minimal parameter capacity if $D$ is small.

\begin{remark}
    The lower bound suggests a gap in the intermediate regime. If $D = R^a$ for some $a < k$, one intuitively expects $W$ to scale smoothly as $\frac{N}{R^a}$. However, our extraction of $R$-ary trees only bounds it at $\frac{N}{R^k}$. We leave tightening this middle regime to future work.
\end{remark}

\subsection{Multi-Hop Constructions without Chain-of-Thought}

To verify the tightness of the lower bounds, we provide two explicit constructions demonstrating the necessary exponential blowup in either embedding dimension or MLP width. 

\begin{theorem}\label{thm:k-hop-non-cot}
    There exists a depth-$k$ transformer that solves the $k$-hop factual recall problem on a set of subjects of size $N$ and a set of relations of size $R$ using one of the following architectures:
    \begin{enumerate}
        \item \textbf{Key-Value Memory:} Embedding dimension $d = O(k+\log N+\log R)$ and MLP width $O(N \cdot R)$. 
        \item \textbf{Embedding Pre-computation:} Embedding dimension $d=O(R^k\log N + k)$ and MLP width $O(R^k\log N)$. 
    \end{enumerate}
\end{theorem}

The full proof can be found in \Cref{appen:proof_k_hop_no_cot}; here, we provide a brief intuition for the proof.
For the \emph{Key-Value Construction}, the embedding dimension is constrained to $O(\log N)$, preventing the embeddings from holding the graph structure. The $k$-layer transformer maintains a single active token that extracts relations sequentially. At each layer, an $O(N \cdot R)$ wide MLP operates as a brute-force lookup table, matching the exact subject-relation pair to output the next subject. For the \emph{Embedding Pre-computation Construction}, expanding the embedding dimension to $\tilde{O}(R^k)$ allows encoding the entire $k$-hop tree of the initial subject directly into its embedding vector. The transformer evaluates the query by traversing this static tree; at each layer $l$, the attention head fetches the current relation $r_l$, and an $\tilde{O}(R^k)$ MLP acts as a Boolean selector to discard irrelevant branches, narrowing the tree until the final answer remains.

\subsection{Breaking the Lower Bound with Chain-of-Thought}

The fundamental limitation of the prior constructions is their reliance on a single forward pass. By incorporating CoT, we circumvent the lower bound of \Cref{thm:k-hop lower bound} and achieve a two-fold advantage: (1) The required network depth reduces from $k$ to $1$ via the autoregressive generation of intermediate outputs, and (2) The capacity requirement on both the model size and embedding dimension is bypassed entirely.

\begin{theorem}\label{thm:k-hop-cot}
    There exists a $1$-layer transformer with CoT, embedding dimension $d = O(R\log N + k)$ and MLP width $O(R\log N)$ that solves a $k$-hop factual recall problem.
\end{theorem}

\begin{table}[h]
    \centering
    \begin{tabular}{lccc}
        \toprule
        \textbf{Mode} & \textbf{Depth} & \textbf{Embedding Dim } & \textbf{MLP Width} \\
        \midrule
        No CoT (Key-Value memory, \Cref{thm:k-hop-non-cot} (1)) & $k$ & $O(k + \log N + \log R)$ & $O(N R)$ \\
        No CoT (Embedding, \Cref{thm:k-hop-non-cot} (2)) & $k$ & $O(R^k \log N + k)$ & $O(R^k \log N)$ \\
        A solution with CoT (\Cref{thm:k-hop-cot}) & $1$ & $O(R \log N + k)$ & $O(R \log N)$ \\
        \bottomrule
    \end{tabular}
    \caption{Expressivity trade-offs for the $k$-hop factual recall problem.}
    \label{tab:cot-tradeoffs}
\end{table}

The full proof can be found in \Cref{appen:proof_k_hop_cot}. The main crux of the proof is that, with CoT, the model only needs the capacity to solve the $1$-hop problem locally. This produces, in essence, a similar construction to \Cref{thm:shared attr} of the $1$-hop warm-up case.
The subject embedding is initialized with an  $O(R\log N)$ representation of its immediate 1-hop neighborhood. During generation, the attention mechanism is configured to route the target relation $r_l$ to the previously generated intermediate subject $s_{l-1}$. The $1$-layer MLP applies the same Boolean gating mechanism as the single-hop case to extract $s_l$. By materializing intermediate states into the sequence, the spatial complexity required to evaluate deep composition is bypassed entirely, matching the optimal $O(R\log N)$ capacity threshold. Crucially, CoT circumvents the capacity lower bound of \Cref{thm:k-hop lower bound} by exploiting the output unembedding operation as a discrete memory fetch. Because autoregressive generation forces a projection back into the vocabulary space at each step $l$, the model dynamically re-queries the static embedding matrix. This allows the network to sequentially retrieve the $1$-hop topological neighborhood of $s_l$, bypassing the need to store the entire $k$-hop evaluation tree within a single continuous hidden state. The trade-offs between the different constructions are summarized in \Cref{tab:cot-tradeoffs}.

\section{Experiments}\label{sec:experiments}

A natural question is whether gradient descent discovers a solution with the same qualitative structure as our construction. We investigate this question through a systematic experimental study in the shared-attribute setting, which is the more challenging and practically relevant case.\footnote{We release the code at \href{https://github.com/shauli-ravfogel/geometric-recall-neurips}{this repository}.} 

\subsection{Setup}\label{sec:exp-setup}

We instantiate the shared-attribute setting with $N = 4096$ subjects (equivalently, attributes) and iterate over the number of relations $R \in \{2, 4, 8, 10, 12, 14, 16\}$ and the embedding dimension $d \in \{32, 64, 128, 256, 512, 768\}$. For each relation~$r$, a random bijection $g_r : [N] \to [N]$ maps subjects to attributes drawn from a shared $N$-attribute pool; the model is given $(s, r, g_r(s))$ and trained as an autoregressive LM. The bijections are sampled uniformly at random (no language model involved in generating them, following the standard protocol in the memorization-capacity literature \citep{nichani2024understanding,cabannesscaling,allenzhu2024knowledge3}). Random facts are the incompressible worst case addressed by the theory. Accuracy is the fraction of all $N \times R$ subject--relation pairs the model classifies correctly. The model is a single-layer transformer with uniform attention and a two-layer GELU MLP with Pre-RMSNorm, using a single embedding pool for subjects and attributes (separate input and output projections).\footnote{We use uniform attention in the single-hop experiments due to the simple sequence structure, matching the theory. Experiments with learned attention yield qualitatively similar results.} We report mean$\,\pm\,$std across three seeds per cell. Full architecture and training hyperparameters are deferred to \cref{app:exp-details}. To isolate the role of \emph{learned} embeddings in the geometric memorization mechanism, we additionally run the entire sweep with the input entity embedding table frozen at its random initialization, training only the attention/MLP weights and the output projection. This is the regime studied by \citet{nichani2024understanding} and serves as a control that reveals which experimental signatures depend on having a learnable representation.

\textbf{Analysis.}
After training, we probe whether the learned solution matches the predicted structure through three complementary analyses (full methodology in \cref{app:exp-analyses}). First, for the \textbf{linear readout}, for each relation $r$, we fit a linear map $W_r$ via ridge regression such that $s_x W_r \approx a_{g_r(x)}$ on held-out subjects; the superposition prediction of \cref{thm:shared attr} implies that all $R$ attributes should be recoverable this way. Second, for \textbf{causal interventions}, we use the per-relation readouts $W_r$ to construct minimum-norm perturbations that swap the attribute for a queried relation, and measure both whether the MLP follows the swap on that relation and whether its predictions on the \emph{other} relations remain stable; the geometric mean of the two scores defines a \emph{selectivity} metric. Third, for \textbf{MLP-freeze transfer}, after training we freeze the model and sample new random bijections $g'_r$. Then, we define $W_{\text{stack}} \;=\; [\,W_1 \mid \cdots \mid W_R\,]$, and reinitialize subject embeddings via $s_x = \text{new target} \cdot W_{\text{stack}}^+$, finding the subject representations that linearly encode the \emph{new} attributes under $W_{\text{stack}}$; high zero-shot accuracy on $g'$ would confirm the MLP is a generic selector rather than a memorizer of specific input--output pairs.

\subsection{Results}\label{sec:exp-results}

\subsubsection{Convergence and Capacity}

The model with learned embeddings reaches perfect memorization for all values of $R$ provided that $d \geq 128$.
The frozen-embedding control also reaches perfect accuracy, but only at substantially higher~$d$ (e.g., $d \geq 512$ is required to memorize $R=16$ relations vs.\ $d \geq 128$ with learned embeddings) --- consistent with the parameter-count bound $\tilde\Omega(NR)$ of \citet{nichani2024understanding}, which the associative memory regime must satisfy without the help of structured embeddings. See \cref{fig:acc-heatmap} in the appendix for the full results.
To quantify the gap, we iterate $N \in \{2^{7}, \ldots, 2^{15}\}$ at fixed $R=8$ and, for each $N$, binary-search the smallest embedding dimension $d_{\min}(N)$ at which trained accuracy crosses $0.95$. We do this in two regimes: trainable embeddings, and frozen random embeddings. We either hold the MLP hidden width fixed at $64$, so any growth of $d_{\min}$ is forced through the embedding geometry rather than through the MLP, or scale the width as $4d$. The results are shown in \cref{fig:scaling}: the trainable curve scales as $d_{\min} \approx a + b \log_2 N$, matching the $\Theta(R \log N)$ rate of \cref{thm:shared attr}, while the frozen curve scales as a power law of $N$ with empirical exponent $\alpha \approx 1$ or $\alpha \approx \frac{1}{2}$ for fixed and variable-width MLP, respectively. This matches the $d_{\min} = \Theta(NR)$ prediction one obtains for fixed random embeddings.

\begin{figure}[t]
    \centering
    \begin{subfigure}{0.35\textwidth}
        \centering
        \includegraphics[width=\linewidth]{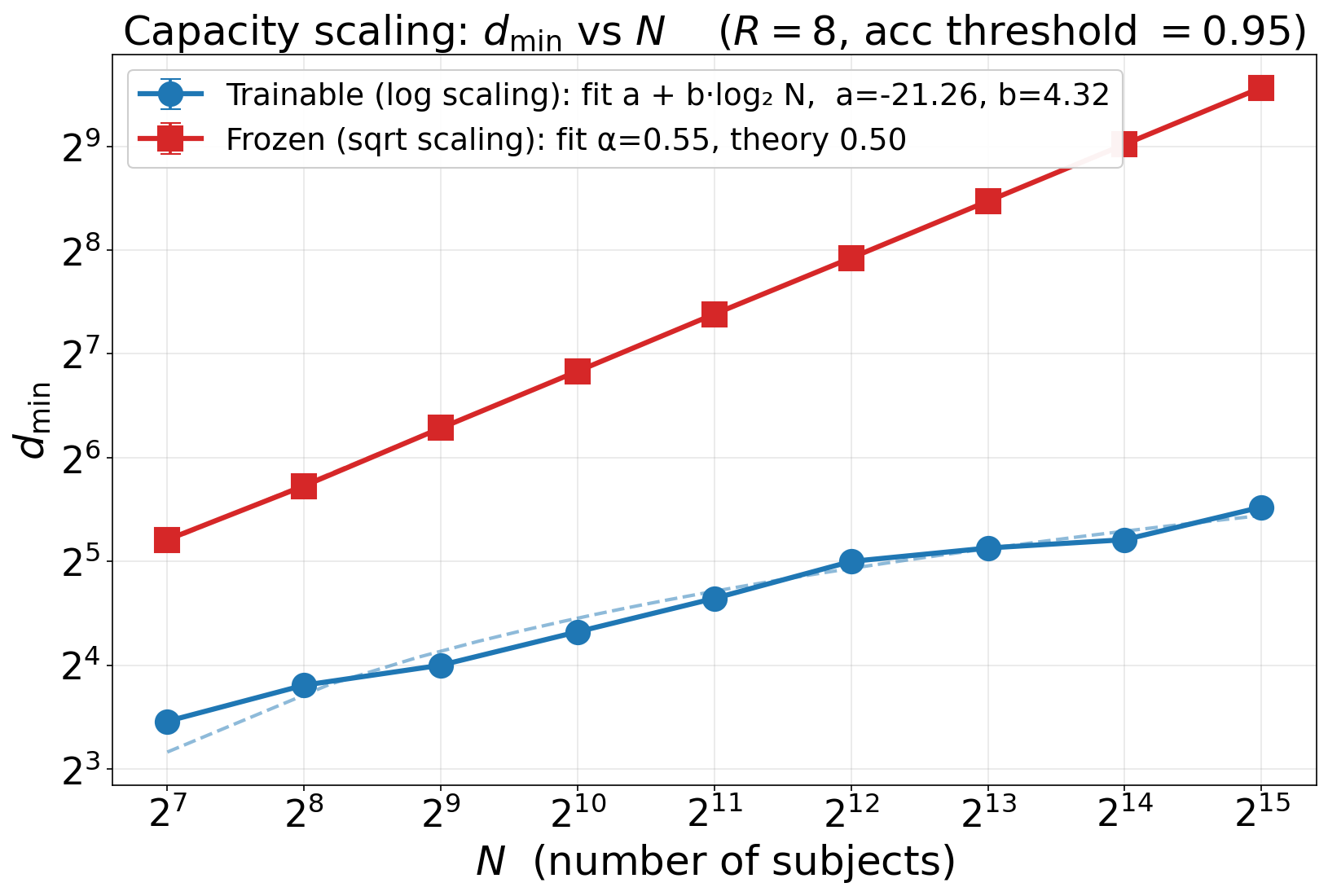}
        \caption{MLP width = $4d$}
    \end{subfigure}
    \hspace{0.02\textwidth}
    \begin{subfigure}{0.35\textwidth}
        \centering
        \includegraphics[width=\linewidth]{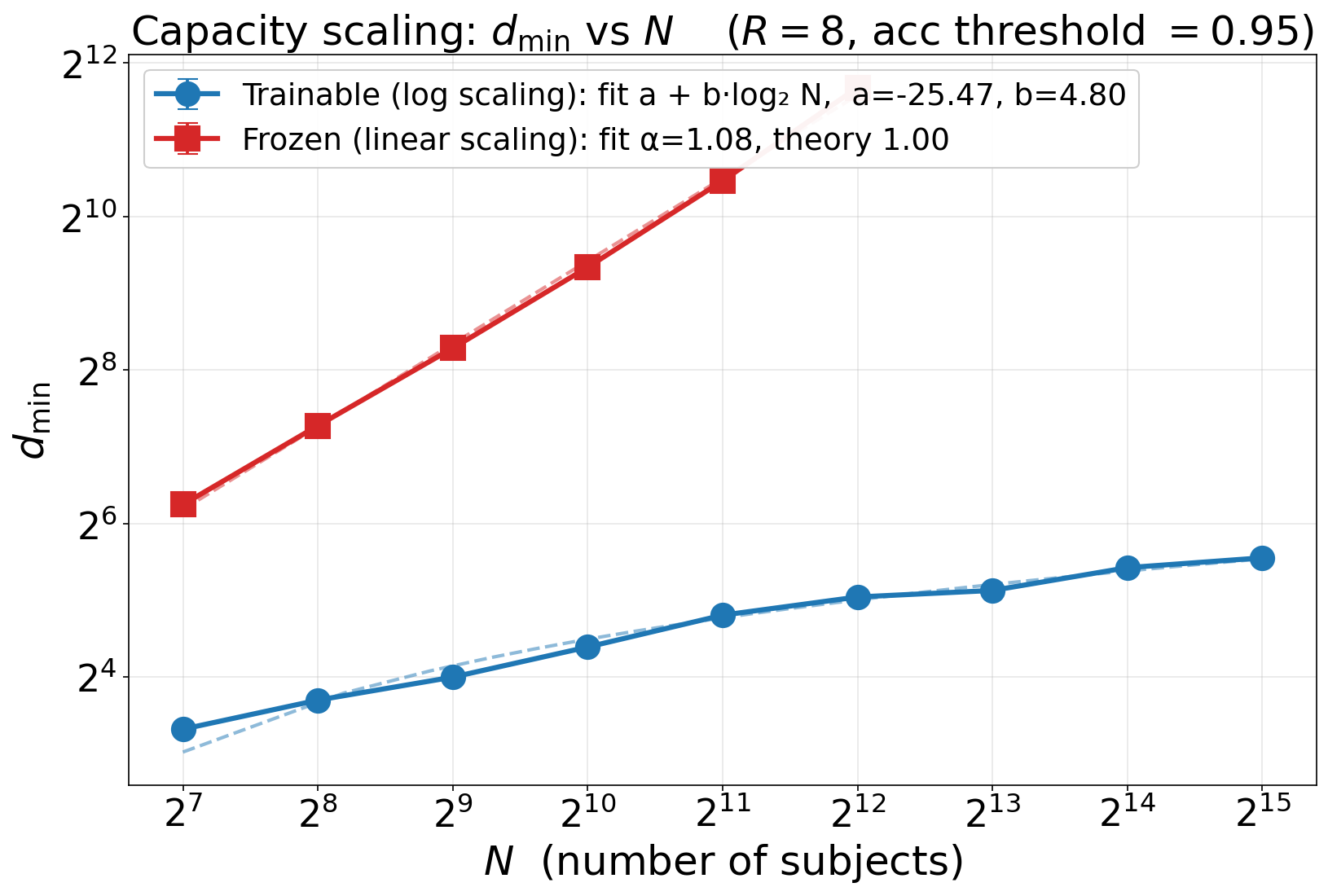}
        \caption{MLP width = 64}
    \end{subfigure}
    \caption{Scaling behavior of factual recall with learned or frozen embeddings.}
    \label{fig:scaling}
\end{figure}

\subsubsection{Linear Structure in the Embeddings}
We find that the embeddings develop linear structure that allows for efficient geometric factual recall.

\textbf{Superposition structure in embeddings.}
\cref{fig:readout-acc} shows the per-relation linear readout accuracy on subject input embeddings and hidden states. With trainable embeddings, the readout is near-perfect across all cells with sufficient capacity, demonstrating that each relation's attribute can be decoded from the subject embedding by a simple per-relation linear map --- exactly the superposition structure predicted by \cref{thm:shared attr}. With frozen embeddings (\cref{fig:frozen-readout}) the readout collapses to chance for the subject embeddings, and is more moderate for the hidden states: by construction, random fixed embeddings cannot encode relation-specific attribute information, so any memorization that occurs must be implemented entirely in the attention/MLP weights. 

\begin{figure}[t]
    \centering
    \includegraphics[width=0.65\textwidth]{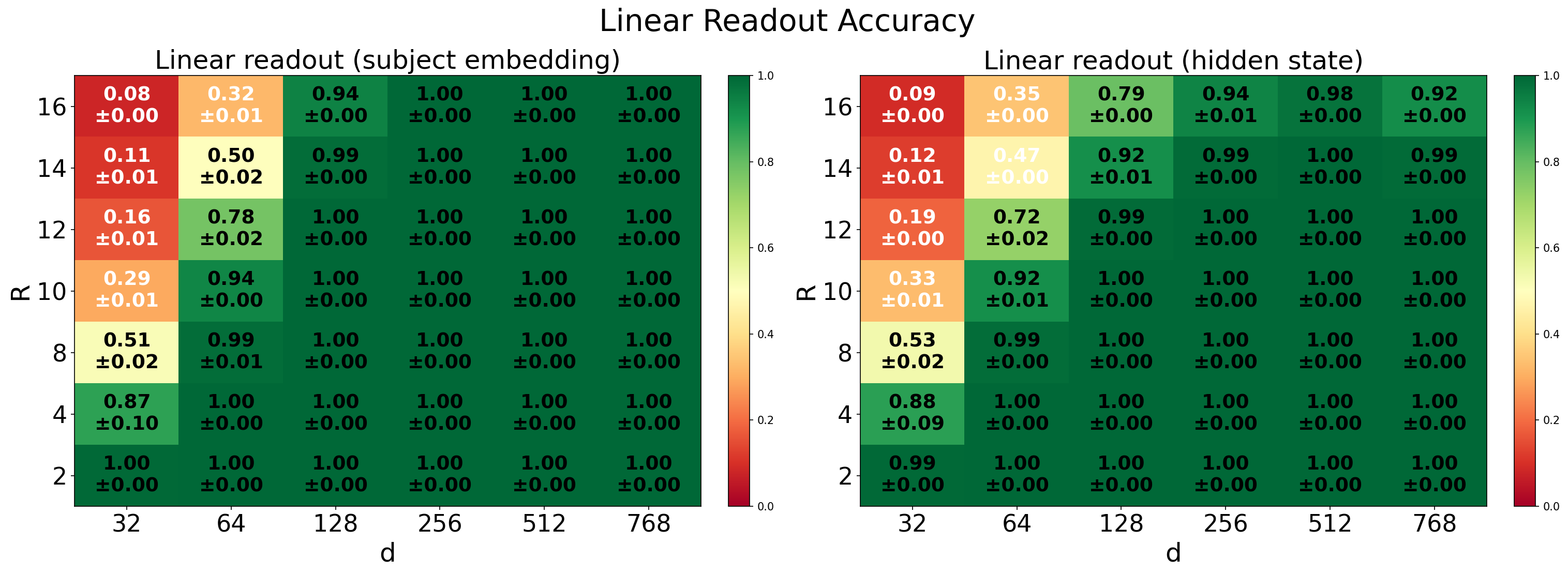}
    \caption{Per-relation linear readout accuracy from the subject embedding $s_x$, and hidden states, with learnable embeddings.}
    \label{fig:readout-acc}
\end{figure}

\textbf{MLP as a generic selector.}
\cref{fig:selectivity} shows the best-$k$ selectivity scores as a function of~$d$ for each~$R$. Across all configurations with sufficient capacity, selectivity is high: the MLP correctly follows counterfactual changes to the queried relation while leaving other relations unaffected. This confirms that the MLP implements a relation-specific selection mechanism, consistent with the ReLU gating construction in the proof of \cref{thm:shared attr}. Notably, the rank truncation matters: using the full pseudoinverse introduces noise from small singular values, while truncating to $k \approx d_{\min}$ yields the cleanest interventions, suggesting that the effective dimensionality of the per-relation readout subspace matches the theoretical prediction. 
The most striking evidence comes from the MLP freeze experiment (\cref{fig:selectivity-freeze}b). After replacing all bijections $g_r$ with fresh random bijections $g'_r$ and initializing subject embeddings to encode the new superposition via the smart initialization, the frozen MLP achieves substantial accuracy \emph{without any retraining}. This demonstrates that the MLP has learned a generic relation-conditioned selector: it extracts the attribute corresponding to the queried relation from whatever superposition is stored in the subject embedding, rather than memorizing specific input--output associations. 

\begin{figure}[h]
    \centering
    \begin{subfigure}[t]{0.35\textwidth}
        \centering
        \includegraphics[width=\linewidth]{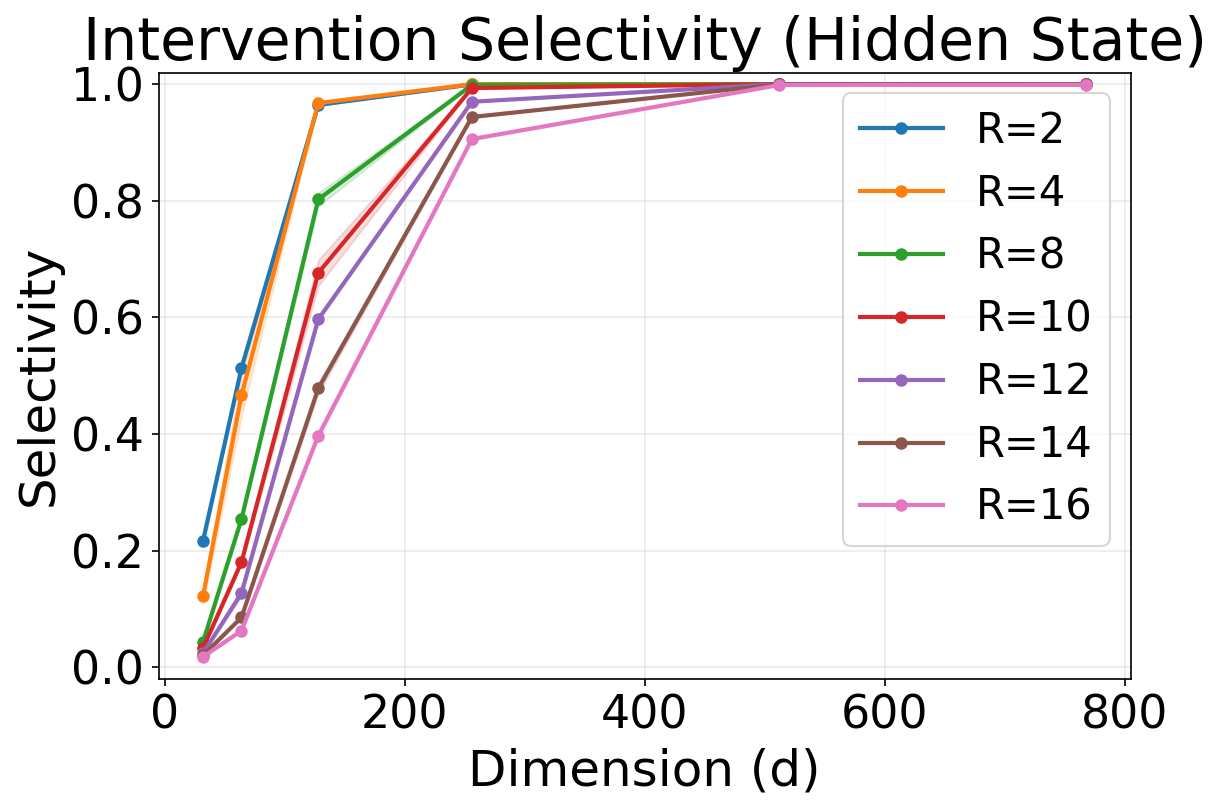}
        \caption{Intervention selectivity scores (on the hidden state) as a function of embedding dimension, for varying~$R$.}
        \label{fig:selectivity}
    \end{subfigure}
    \hspace{0.02\textwidth}
    \begin{subfigure}[t]{0.35\textwidth}
        \centering
        \includegraphics[width=\linewidth]{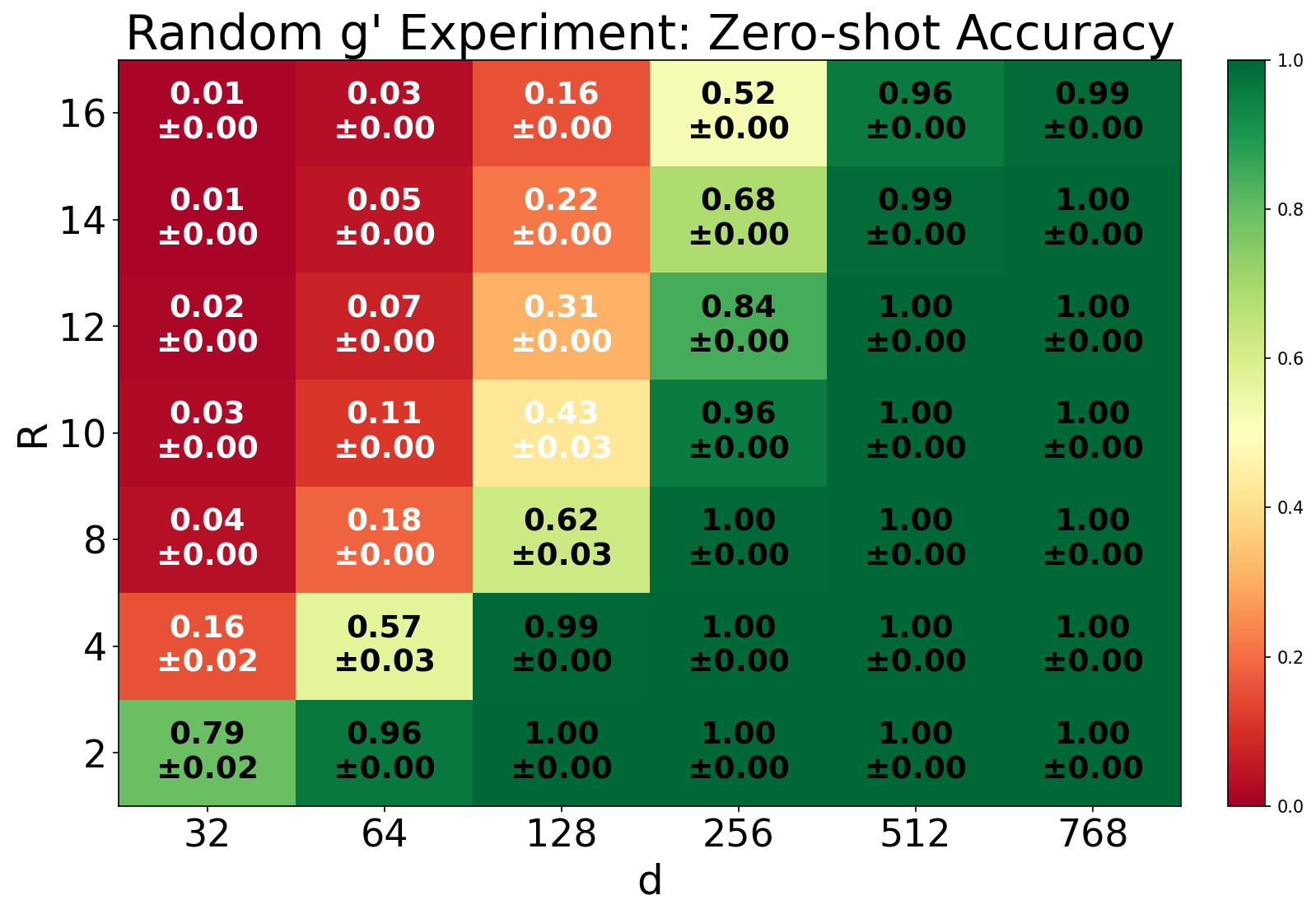}
        \caption{Zero-shot accuracy with smart initialization on new bijections~$g'$.}
        \label{fig:freeze-zeroshot}
    \end{subfigure}
    \caption{Intervention-based evidence for the generic selection mechanism learned by the MLP.}

    \label{fig:selectivity-freeze}
\end{figure}

\paragraph{Multi-hop experiments.} See \cref{app::exp-multihop} for experiments on the multi-hop setting. The findings  show that CoT and trainable embeddings indeed facilitate compositional factual recall.

\subsubsection{Below Capacity: Errors as Geometric Interference}\label{sec:exp-subcapacity}

Our constructions characterize the perfect-memorization regime, following prior theoretical work on factual recall, but they also predict \emph{how} recall fails below capacity. Our construction allocates $\Theta(\log N)$ dimensions to each of the $R$ relation blocks, so that the $N$ attribute components within a block remain near-orthogonal (\cref{lem:orth vectors pm 1}). When $d$ falls below $\Theta(R \log N)$, the blocks must share dimensions, and the readout should return a \emph{mixture} (superposition) of stored attributes. This predicts two error families: \emph{gating failures}, where the model outputs $g_{r'}(s)$ for some $r' \neq r$ (right subject, wrong relation), and \emph{crowding}, where it outputs the attribute of a geometrically nearby subject, or an attribute whose encoding lies near the correct one.

We trained models at sub-capacity ($N = 4096$, $R = 16$) and classified every error (\cref{app:subcapacity}). For $d \in \{24, 32, 48\}$, 84--89\% of errors fall into a predicted family. They are dominated by gating failures (78--84\% of errors, vs.\ 0.37\% under chance): the model retrieves the right subject's superposition, but the relation gate cannot choose the correct component, so the outputs are mixtures. At $d = 48$, the softmax mass splits almost evenly between the correct attribute and the same subject's other attributes, and on errors the correct attribute is typically the runner-up (median rank 3 of 4096). Frozen-embedding models at matched accuracy show much less of this structure (e.g., 39\% vs.\ 84\% gating failures at accuracy ${\approx}\,0.57$), and their subject-neighbor crowding is near chance. The error pattern itself therefore distinguishes geometric from associative memorization. Under-capacity errors in this setting are thus not random guesses but confusions with structurally adjacent facts, a pattern reminiscent of factual hallucinations in LLMs (where the model might confuse a capital city with \emph{another} capital city, for example). Testing this connection in pretrained models is left to future work.

\subsection{Probing Real Language Models for Linear Relational Encoding}\label{sec:real-lm-probing}

Our theory (\cref{sec:fact_recall_theory}) predicts that subject embeddings encode each relation's answer as a \emph{linearly decodable} component, thereby potentially providing the missing theoretical basis for the linear relational structure observed empirically by \citet{hernandez2023linearity} and \citet{merullolinear}. We test this on real pretrained LMs. Two caveats: (i) the theory targets a single-layer transformer, so we probe at \emph{every} layer rather than committing to one; (ii) entities span multiple BPE tokens, so following \citet{meng2022locating, geva2023dissecting} we read $\mathbf{s}_\ell$ at the \emph{last subject token}.

\textbf{Setup.}
Unlike the random bijections of \cref{sec:exp-setup}, this experiment uses real-world facts. For each relation $r$ in each category $c$ (real-world entities and their documented attributes, compiled with Claude under a confidence-or-skip filter; see \cref{app:real-lm-probing-data}), we fit a low-rank affine probe from the layer-$\ell$ subject hidden state $\mathbf{s}_\ell$ to the LM-head row of the gold answer's first token $D_{[t_o]}$, using a cosine objective: $\mathcal{L}_r = 1 - \cos(W_r \mathbf{s}_\ell + b_r,\, D_{[t_o]})$, with $W_r = A_r B_r$ a rank-$k=512$ factorization. The cosine objective is scale-invariant, sidestepping the magnitude-collapse we observed with squared error against rows of $D$. Probes are trained per (category, relation) on a subject-disjoint 80/20 split, and Hits@1 ranks $W_r \mathbf{s}_\ell + b_r$ against every row of $D$. We sweep five LMs (Qwen2.5-0.5B, Qwen3-14B, Llama-3.1-8B, Llama-3.2-1B, Phi-4) and five entity categories (\textit{people}, \textit{companies}, \textit{films}, \textit{buildings}, \textit{programming languages}). For each entity category we have several appropriate relations, such as \emph{birth place} and \emph{occupation} for the \emph{person} category.
\citet{hernandez2023linearity} previously tested linear encoding of relations, but targeted the hidden state leading to the attribute rather than the output embeddings, relied on the Jacobian based on a few training examples, and reported a \emph{lack} of linear encoding for the latter. Our larger per-relation data and cosine objective allow us to show output embeddings \emph{are} linearly encoded.

\textbf{Findings.}
We find a linearly decodable signal in subject hidden states across all five LMs and five categories. Per-relation Hits@1 substantially exceeds random-vocabulary chance ($1/V \approx 10^{-5}$) and a constant majority baseline, with best-layer MRR (averaged over models) reaching $0.69$ on \textit{people}, $0.58$ on \textit{companies}, $0.56$ on \textit{buildings}, $0.55$ on \textit{programming languages}, and $0.44$ on \textit{films} (full per-(model, category) breakdowns in \cref{app:real-lm-probing}). This encoding is not necessarily linear at layer 0: input embeddings carry some signal, but are considerably weaker than mid-late layers (e.g.\ for \textit{people}, layer-0 MRR is $0.44$ vs.\ $0.69$ at the best layer)---and Hits@1 rises through middle layers and plateaus, with the optimal layer roughly 50--80\% through the network, suggesting the structure our theory predicts emerges via layer-by-layer enrichment, consistent with \citet{geva2023dissecting}. Crucially, our probe targets a strictly more demanding object than \citet{hernandez2023linearity}, who use an affine map from $\mathbf{s}$ to the last-layer hidden state $\mathbf{o}$ (decoded by $D$): we predict the LM-head row $D_{[t_o]}$ directly from $\mathbf{s}$ and show this is feasible, indicating that subject hidden states encode \emph{output-vocabulary} directions for their attributes linearly, not merely a route into a non-linearly-decoded intermediate state.

\textbf{External benchmark.}
To rule out artifacts of our dataset compilation, we replicate the analysis on T-REx \citep{elsahar2018trex}, a human-curated Wikipedia--Wikidata benchmark with no language-model involvement. We use the T-REx relations that match ours and, following \citet{hernandez2023linearity}, restrict to facts each LM knows (gold first token within its top-3 predictions). Across Llama-3.2-1B, Llama-3.1-8B and Qwen3-14B, best-layer Hits@1 aggregated over relations is 0.53--0.55, 3--4$\times$ the majority baseline (0.14--0.18) and in the range of \cref{tab:real-lm-best-mrr}. The layer-wise profile also replicates, with decodability rising from a non-trivial layer-0 value, peaking at about 75\% of depth in all three models, and dipping slightly at the final layer (\cref{app:real-lm-probing-trex}).

\section{Discussion}

We develop a theoretical and empirical account of \emph{geometric} factual memorization in transformers. In a controlled setting where a single-layer transformer memorizes random subject--attribute bijections, we prove that dimension $d = O(R \log N)$ suffices: subject embeddings store linear superpositions of attribute vectors, while a small relation-conditioned ReLU MLP selects the relevant attribute. We extend this to multi-hop reasoning, proving a capacity--depth tradeoff without chain-of-thought (\cref{thm:k-hop lower bound,thm:k-hop-non-cot}) and showing that chain-of-thought reduces $k$-hop reasoning to the single-hop case (\cref{thm:k-hop-cot}). Empirically, gradient descent recovers the predicted structure. We show that attributes are linearly encoded, and use causal interventions to demonstrate a simple selection mechanism over this linear encoding. 

\paragraph{Limitations.} Our capacity experiments use synthetic random bijections. This setting is the incompressible worst case matching the theory, but it abstracts away factors such as correlations, typing, and frequency effects in real data, and mechanisms verified causally in the controlled setting are established only based on correlations in pretrained LMs (\cref{sec:real-lm-probing}). More broadly, our setting treats relations as opaque tokens and attributes as arbitrary entities, while real factual knowledge carries rich semantic structure, such as correlated relations, typed attributes, polysemy, multi-token entities, and long-tailed frequencies. Extending the geometric account to data with such regularities, where the optimal embedding geometry should reflect them rather than encode each fact independently, would bring the theory closer to natural language. Additionally, generalizing the single-layer constructions to deep models can shed light on the implementation of factual recall in pretrained LLMs.

\section*{Acknowledgements}
This work was supported in part through the NYU IT High Performance Computing resources, services, and staff expertise, and by the National Science Foundation (NSF) under Grant No.
IIS-2239862 to TL. We thank Yanai Elazar for his valuable comments. 

\bibliographystyle{plainnat}
\bibliography{refs}

\clearpage
\newpage

\appendix
\section*{Appendix}
\crefalias{section}{appendix}
\crefalias{subsection}{appendix}
\crefname{appendix}{Appendix}{Appendices}
\Crefname{appendix}{Appendix}{Appendices}

\section{Additional Related Work}\label{app:extended-related}

\textbf{Empirical capacity of factual storage.}
\citet{roberts2020much} introduced closed-book QA as a capacity probe and showed that storable knowledge scales with model size. The most direct empirical analog of our capacity claims is the Physics of Language Models series: \citet{allenzhu2023knowledge1} isolate when memorized facts become extractable, \citet{allenzhu2023knowledge2} document that simple manipulations of stored facts demand more capacity than retrieval, and \citet{allenzhu2024knowledge3} establish a $\sim$2-bits-per-parameter scaling law.

\textbf{Tensor decompositions for relational data.}
The multi-relation fact-storage problem has a long history in knowledge-base completion as a tensor-decomposition problem \citep{nickel2011three, yang2014embedding, trouillon2016complex, lacroix2018canonical}. CP, DistMult, and ComplEx factorize the $N \times R \times N$ relation tensor as a sum of rank-one terms. For random bijections, each relation slice is a permutation matrix of matrix rank $N$, and the worst-case CP rank for representing the full tensor is $O(NR)$ \citep{trouillon2017knowledge}, which is substantially above our $O(R \log N)$ bound when $R \ll N$. Beyond this quantitative gap, CP-style methods score triples with a fixed trilinear form, which no standard transformer implements. Our question is what mechanism a transformer --- with attention and MLPs but no built-in trilinear product --- uses to solve the same task. More fundamentally, the transformer is a general-purpose computational system, while CP is a fixed scoring function tailored to single-hop link prediction. This generality is what lets us study compositional phenomena, such as multi-hop recall, the interaction between chain-of-thought and embedding capacity, and the depth--width tradeoffs of multi-hop question answering, which have no natural analog in tensor decomposition.

\textbf{Memorization capacity of neural networks and transformers.}
A long line of work establishes worst-case memorization bounds: $\tilde{O}(\sqrt{N})$ parameters suffice for ReLU networks to memorize $N$ arbitrary points \citep{vardi2022optimal}, with logarithmic width sufficing in the robust regime \citep{egosi2025logarithmic}. For transformers, \citet{yun2020transformers} prove universal approximation via quantize-and-memorize, and subsequent work tightens the constants for arbitrary sequence-to-sequence memorization \citep{kim2023provable, mahdavi2024memorization, kajitsuka2024transformers}. These bounds depend on the number of examples and sequence length and do not exploit relational structure. Our setting --- bijective subject--relation--attribute mappings --- is the standard model used in works on real-world factual recall \citep{nichani2024understanding, allenzhu2024knowledge3, ravfogel2025emergence}.

\section{Experimental Details}\label{app:exp-details}

This appendix collects the full hyperparameter and methodology. All experiments were conducted on a single H200 node. 

\subsection{Architecture}\label{app:exp-arch}

All single-hop and multi-hop experiments share the same backbone: a single-layer transformer with one attention head and a two-layer MLP with GELU activations and expansion ratio~4, using Pre-LayerNorm with RMSNorm. Subjects and attributes draw from a single embedding pool of size~$N$, with separate input and output projections (i.e.\ the input embedding $E_{\text{in}}$ and the output unembedding $E_{\text{out}}$ are independent matrices. Both are built from a shared $N$-token entity vocabulary, where the same token can serve as a subject in one instance and an attribute in another). The single-hop runs use uniform fixed attention, whereas the multi-hop runs use a single \emph{learned} attention head together with learned positional embeddings.

\subsection{Single-Hop Training}\label{app:exp-singlehop}

Each model is trained with AdamW (learning rate~$1.0$, weight decay~$0.1$) for $15{,}000$ steps with batch size~$1024$ and gradient clipping at norm~$1$, with early stopping when training loss falls below~$10^{-4}$. We run three seeds per $(d, R)$ configuration and report mean~$\pm$~standard deviation throughout. The frozen-embedding control uses identical settings except that the input entity embedding table is held fixed at its random initialization while the attention/MLP weights and the output projection are trained as usual.

\subsection{Analysis Methodology}\label{app:exp-analyses}

This subsection gives the definitions for each of the three post-training analyses summarized in \cref{sec:exp-setup}.

\textbf{Linear readout.}
For each relation $r$, we fit a linear map $W_r$ via ridge regression such that $s_x W_r \approx a_{g_r(x)}$, and evaluate classification accuracy on held-out subjects. If subject embeddings encode a superposition of all $R$ attribute vectors as predicted by \cref{thm:shared attr}, a per-relation linear readout should recover each one with high accuracy.

\textbf{Causal interventions.}
Using the per-relation readout maps $W_r$, we perform counterfactual perturbations on the subject embedding: we swap the attribute for a queried relation~$r$ via the minimum-norm perturbation $\delta = \Delta a \cdot W_r^+$, where $\Delta a$ is the difference between the substituted and original attribute vectors and $W_r^+$ is the (rank-truncated) Moore--Penrose pseudoinverse. We measure two quantities: (A)~whether the MLP's prediction \emph{follows} the change for relation~$r$, and (C)~whether predictions for all other relations remain \emph{stable}. Their geometric mean,
\[
    \text{selectivity} \;=\; \sqrt{A \cdot C},
\]
captures whether the MLP acts as a relation-specific selector. Since the pseudoinverse is sensitive to small singular values, we sweep over rank-$k$ truncations of $W_r$ and report the best selectivity per cell.

\textbf{MLP freeze experiment.}
After training, we freeze all parameters (and, in particular, the MLP) except the subject embeddings, and sample fresh random bijections $g'_r$. We then reinitialize the subject embeddings either (a)~randomly (the baseline) or (b)~via a ``smart initialization'' that solves for the superposition encoding the new attributes:
\[
    s_x \;=\; \text{target} \cdot W_{\text{stack}}^+,
    \qquad
    W_{\text{stack}} \;=\; [\,W_1 \mid \cdots \mid W_R\,],
    \qquad
    \text{target} \;=\; \mathrm{concat}_r\bigl(a_{g'_r(x)}\bigr).
\]

Which is the solution to:

\[
    s_x^\star \;=\; \arg\min_{s} \|s\|_2 \quad \text{s.t.} \quad s\, W_{\text{stack}} = \text{target}
\]

If the MLP has learned a generic relation-conditioned selector, the smart initialization should yield high \emph{zero-shot} accuracy on $g'$ without any retraining. If, instead, the MLP itself stores the subject--attribute mapping (as implicitly assumed by previous algebraic associative-memory accounts), then zero-shot accuracy on the new bijections should remain at chance.

\section{Additional Experimental Results}\label{app:additional-experiments}
In \cref{fig:acc-heatmap} we show the final accuracy of the models trained in \cref{sec:experiments} with frozen or trainable embeddings.

\begin{figure}[h]
    \centering
    \includegraphics[width=0.8\textwidth]{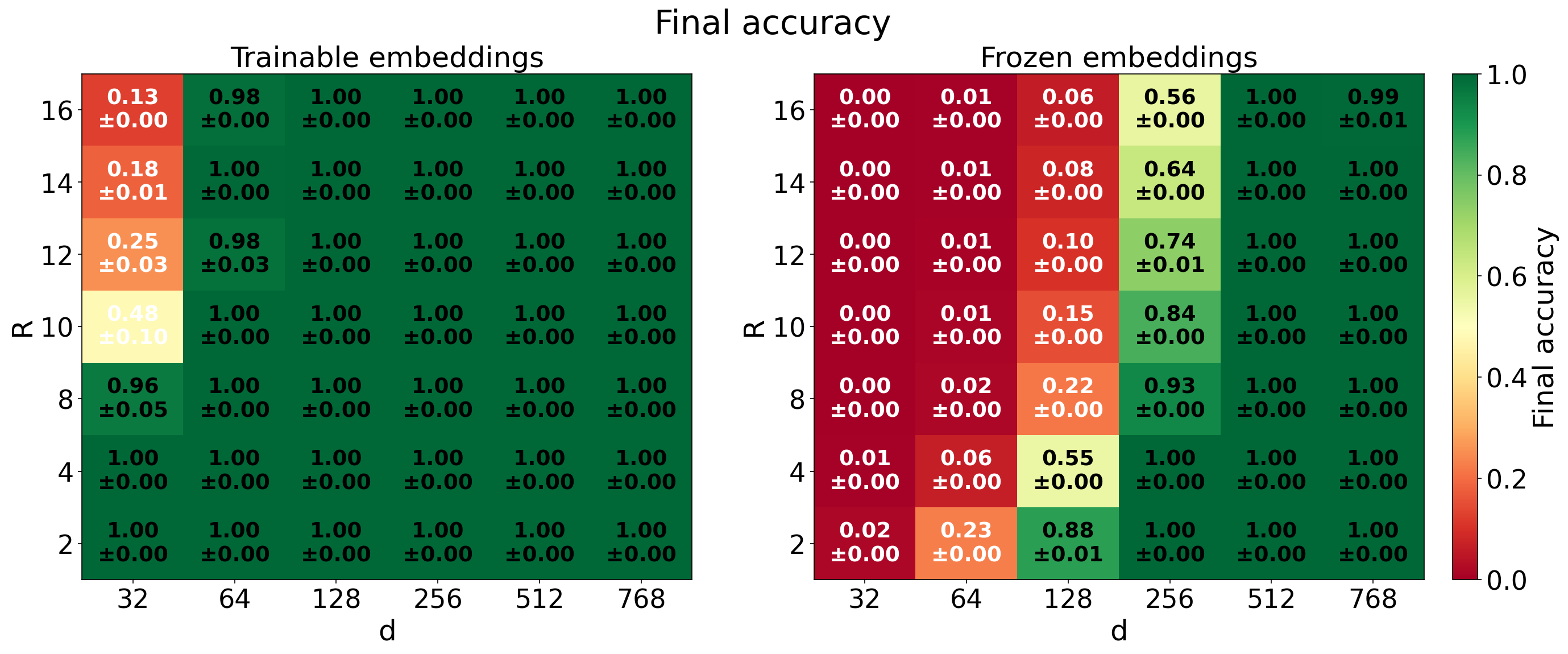}
    \caption{Final accuracy after training, across number of relations~$R$ and embedding dimension~$d$.}
    \label{fig:acc-heatmap}
\end{figure}

\subsection{Multi-Hop Experiments}\label{app::exp-multihop}

We complement the single-hop experiments above with experiments on the multi-hop setting. Recall that our theoretical constructions predict a fundamental capacity--depth tradeoff: without chain-of-thought (CoT), solving $k$-hop queries requires either larger embeddings ($d = O(R^k \log N)$) or a wide MLP ($\text{width} = O(N \cdot R)$), whereas with CoT, a single-layer transformer with $d = O(R \log N + k)$ suffices for any number of hops.

\begin{figure}[h]
    \centering
    \includegraphics[width=0.85\textwidth]{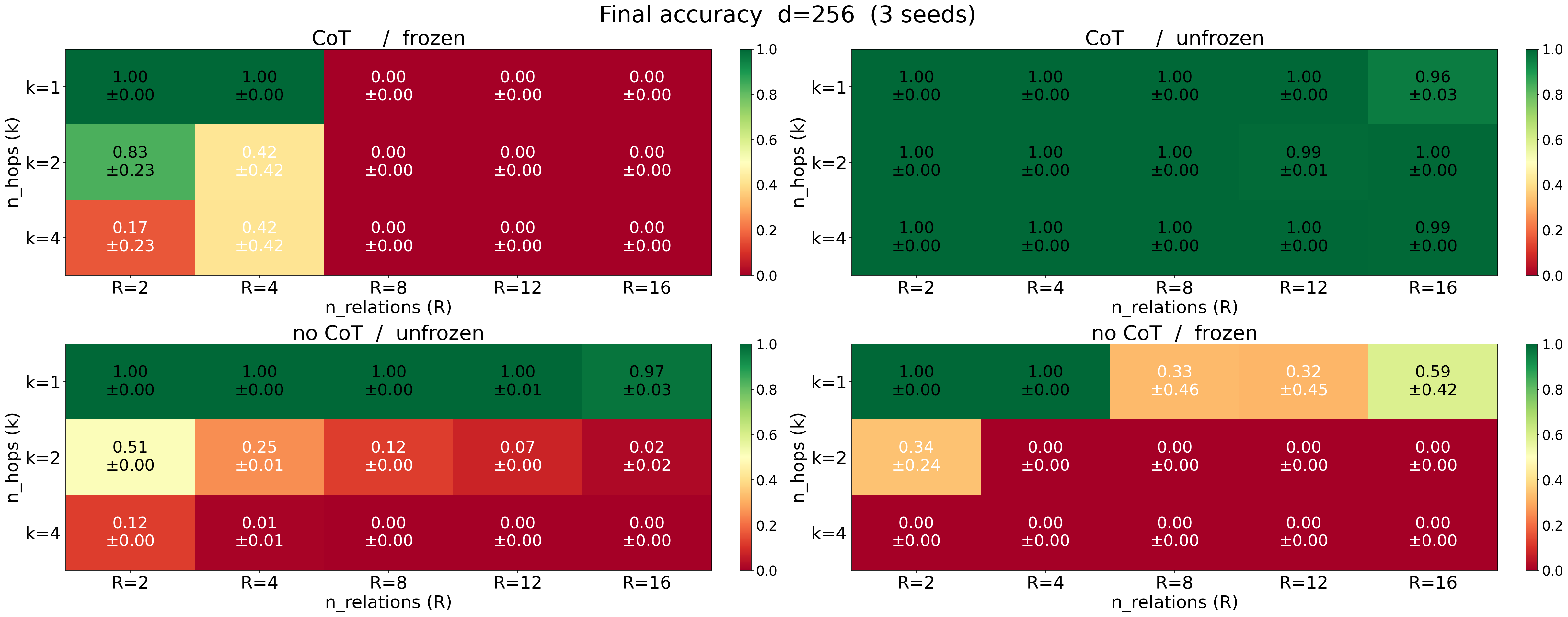}
    \caption{Multi-hop accuracy across number of hops~$k$ and relations~$R$ for $N = 2048$, $d = 256$, averaged over 3 seeds (cells show mean~$\pm$~std; teacher-forced evaluation).}
    \label{fig:multihop-acc}
\end{figure}

\paragraph{Setup.}
We fix $N = 2048$ subjects and embedding dimension $d = 256$, and sweep over $k \in \{1, 2, 4\}$ hops and $R \in \{2, 4, 8, 12, 16\}$ relations. The architecture is the same single-layer transformer used in the single-hop experiments, with one \emph{learned} attention head and learned positional embeddings. We compare four conditions in a $2 \times 2$ design: with or without chain-of-thought (CoT), and with embeddings unfrozen (learned) or frozen at random initialization. Three seeds per cell, mean~$\pm$~std reported. CoT sequences have the form $[\langle\text{seq}\rangle, s_0, r_1, \ldots, r_k, \langle\text{cot}\rangle, s_0, r_1, y_1, r_2, y_2, \ldots, r_k, y_k]$, trained with standard next-token cross-entropy, so all intermediate subjects are supervised. Accuracies in \cref{fig:multihop-acc} are computed with a teacher-forced prefix (ground-truth intermediates in context). When we evaluate representative cells fully autoregressively (average over 3 seeds), accuracy is essentially unchanged in the CoT/unfrozen regime (e.g., $0.998$ vs.\ $0.991$ for $k=4$, $R=16$; $0.993$ vs.\ $0.985$ for $k=2$, $R=16$), while it drops further in the already-failing frozen regime (e.g., $0.33$ vs.\ $0.07$ for $k=4$, $R=4$), consistent with a multiplicative compounding of errors across hops.

Each model is trained with AdamW (learning rate~$10^{-2}$, weight decay~$0.1$) for up to $15{,}000$ steps with batch size~$1024$ and gradient clipping at norm~$1$. Training stops early when either eval accuracy reaches $0.999$ or training loss falls below $10^{-4}$ for three consecutive log checks. We run three seeds per $(k, R, \text{CoT}, \text{frozen})$ cell and report mean~$\pm$~std.

\paragraph{Results.}
\cref{fig:multihop-acc} shows final accuracy across all $(k, R)$ configurations for each condition. The results reveal a clear pattern consistent with the theoretical predictions. With CoT and learned embeddings (top-right panel), the model achieves near-perfect accuracy across all configurations, including $k = 4$ hops with $R = 16$ relations. Without CoT (bottom-left, learned embeddings), performance degrades sharply as $k$ and $R$ increase: even at $d = 256$, the embedding dimension becomes insufficient to encode the exponentially growing answer space. Freezing the embeddings (top-left and bottom-right) collapses performance in nearly all cells: with no learned embeddings, even with CoT the model can only solve the easiest cases ($k = 1$ and very small $R$), confirming that learnable embeddings are essential for efficient multi-hop reasoning. These results provide empirical evidence for the capacity--depth tradeoff predicted by theory: CoT enables efficient multi-hop computation by distributing the reasoning across autoregressive steps, circumventing the exponential embedding bottleneck.

\section{Frozen-Embedding Control: Full Results}\label{app:frozen-results}

To isolate the role of \emph{learnable} embeddings in the geometric memorization mechanism, we re-ran the entire single-hop sweep of \cref{sec:experiments} with the input entity embedding table frozen at its random initialization. All other training hyperparameters and analyses are identical to the trainable-embedding runs. This appendix collects the key figures from that control sweep.

\paragraph{Capacity threshold under frozen embeddings.}
\cref{fig:frozen-acc} shows the final accuracy heatmap. The frozen model still reaches perfect memorization, but only at substantially higher embedding dimension than the trainable-embedding model in \cref{fig:acc-heatmap}: e.g., $R=16$ relations require $d \geq 512$ to memorize, vs.\ $d \geq 128$ when embeddings are learned. This is consistent with the parameter-count bound $\tilde\Omega(NR)$ of \citet{nichani2024understanding}, which the brute-force MLP regime must satisfy without the help of structured embeddings.

\begin{figure}[h]
    \centering
    \includegraphics[width=0.5\textwidth]{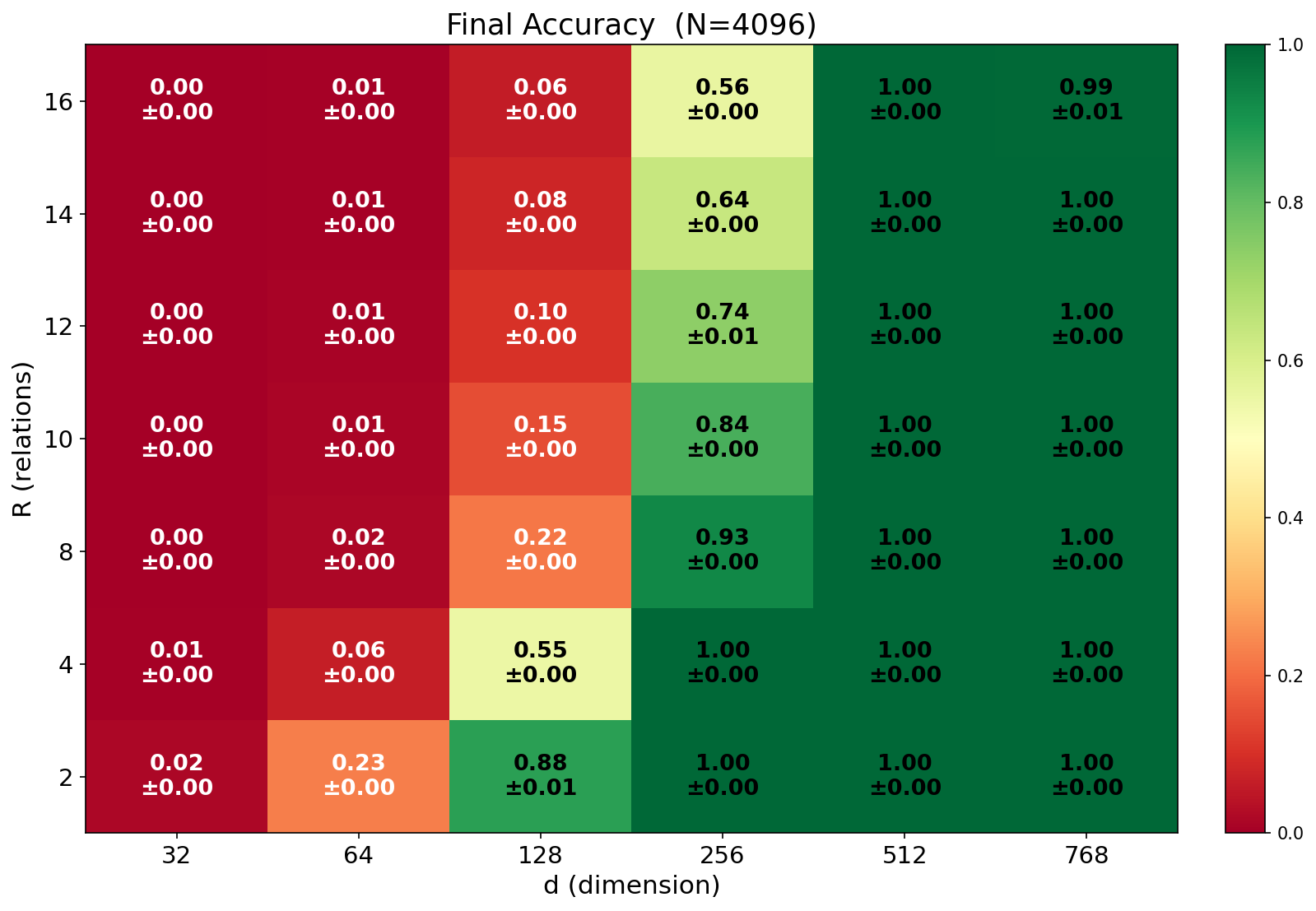}
    \caption{Final accuracy after training, frozen-embedding control. Compare with \cref{fig:acc-heatmap} (trainable embeddings).}
    \label{fig:frozen-acc}
\end{figure}

\paragraph{Linear readout under frozen embeddings.}
\cref{fig:frozen-readout} shows the per-relation linear readout accuracy from both the subject embedding~$s_x$ (left) and the post-attention hidden state~$h$ (right). The $s_x$-readout is at chance throughout the grid: random fixed embeddings do not encode relation-specific attribute information by construction. The $h$-readout is also low across the grid.

\begin{figure}[h]
    \centering
    \includegraphics[width=0.85\textwidth]{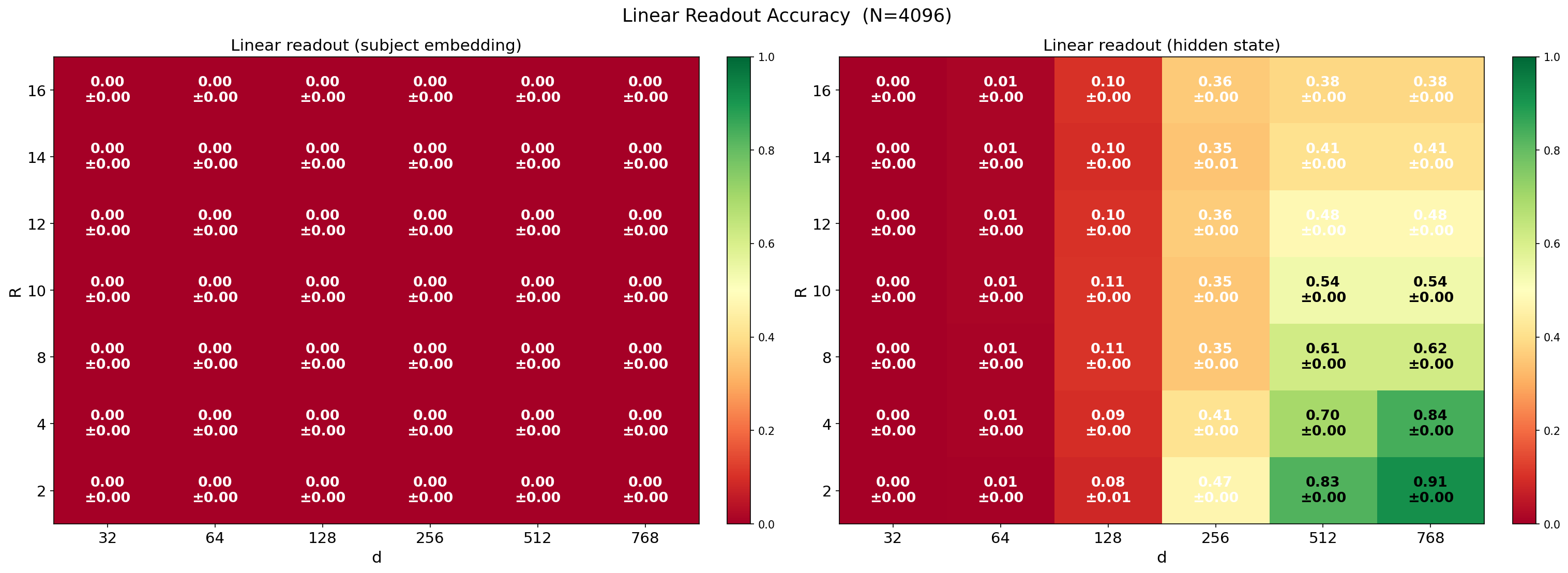}
    \caption{Linear readout accuracy under frozen embeddings, from $s_x$ (left) and from the post-attention hidden state $h$ (right). Both collapse to near-chance, in contrast to the trainable-embedding result of \cref{fig:readout-acc}.}
    \label{fig:frozen-readout}
\end{figure}

\paragraph{MLP freeze experiment under frozen embeddings.}
\cref{fig:frozen-freeze-panel} shows the four-condition freeze panel for the frozen-embedding model. The smart-initialization zero-shot column decreases in performance. With no $W_r$ structure to utilize, the smart-init formula $s_x = \text{target}\cdot W_{\text{stack}}^+$ yields essentially random subject embeddings, and zero-shot accuracy on the new bijections~$g'$ is near chance. Retraining the subject embeddings does recover high accuracy in some settings, as it just re-encodes the new $g'$ into the embedding table directly. 

\begin{figure}[h]
    \centering
    \includegraphics[width=0.99\textwidth]{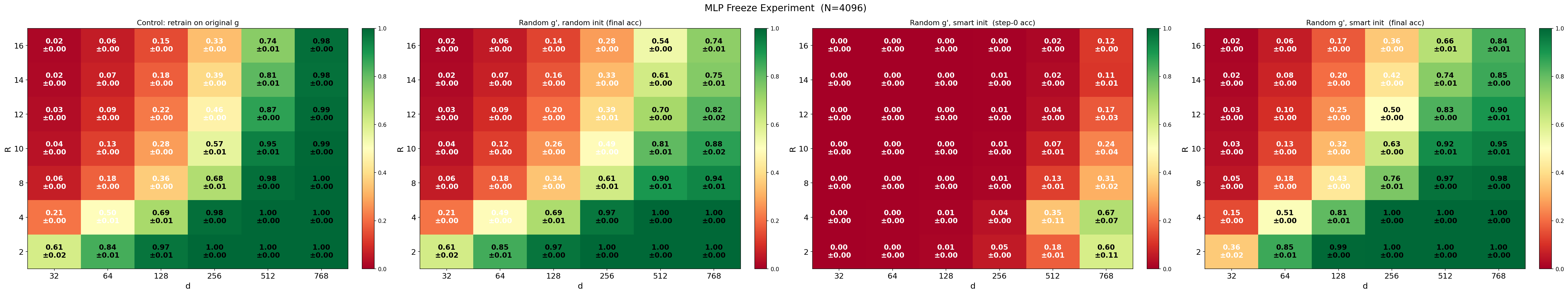}
    \caption{MLP freeze experiment, frozen-embedding control: control (retrain on original~$g$), random-init retrain on~$g'$, smart-init zero-shot, and smart-init after retraining. The smart-init zero-shot column collapses, confirming that the smart-initialization formula does not hold without learned per-relation readouts. }
    \label{fig:frozen-freeze-panel}
\end{figure}

\section{Sub-Capacity Errors}\label{app:subcapacity}

\paragraph{Setup.} We use the single-hop setup of \cref{sec:exp-setup} with $N = 4096$ and $R = 16$. We use embedding dimensions $d \in \{16, 24, 32, 48, 64\}$ with learned embeddings, and $d \in \{128, 256\}$ with frozen embeddings, chosen so that their accuracies match learned-embedding cells. We evaluate all $NR$ queries and classify each error $\hat{y} \neq g_r(s)$ into three categories:
\begin{itemize}
    \item \textbf{Gating failure:} $\hat{y} = g_{r'}(s)$ for some $r' \neq r$. The model returns another relation's attribute for the correct subject.
    \item \textbf{Subject crowding:} $\hat{y} = g_r(s')$ for some $s'$ among the 10 nearest neighbors of $s$ in embedding space. The model returns the correct relation's attribute for a nearby subject.
    \item \textbf{Attribute crowding:} $\hat{y}$ is among the 10 nearest neighbors of $g_r(s)$ in embedding space.
\end{itemize}
Under a uniformly random wrong prediction, the chance rates are $(R-1)/(N-1) \approx 0.37\%$ for gating failures and $10/(N-1) \approx 0.24\%$ for each crowding category. The chance median rank of the correct attribute is ${\approx}\,N/2$.

\begin{table}[h]
\centering
\small
\begin{tabular}{lcccccc}
\toprule
Embeddings & $d$ & Acc. & Gating $g_{r'}(s)$ & Subj.\ crowd. & Attr.\ crowd. & Med.\ rank \\
\midrule
learned & 16  & 0.03 & 41\% & 4.6\% & 3.7\% & 114 \\
learned & 24  & 0.06 & 78\% & 3.2\% & 2.9\% & 36 \\
learned & 32  & 0.13 & 81\% & 3.0\% & 2.7\% & 13 \\
learned & 48  & 0.59 & 84\% & 2.7\% & 2.6\% & 3 \\
learned & 64  & 0.99 & 70\% & 3.0\% & 1.3\% & 2 \\
\midrule
frozen  & 128 & 0.07 & 10\% & 0.5\% & 1.1\% & 124 \\
frozen  & 256 & 0.55 & 39\% & 0.5\% & 1.1\% & 7 \\
\midrule
\textit{chance} & -- & -- & \textit{0.37\%} & \textit{0.24\%} & \textit{0.24\%} & \textit{${\approx}$2048} \\
\bottomrule
\end{tabular}
\caption{Classification of errors below capacity ($N = 4096$, $R = 16$). Rates are fractions of \emph{errors}, not of all queries. Crowding uses the 10 nearest neighbors. The median rank is that of the correct attribute among all $N$ attributes, computed on errors.}
\label{tab:subcapacity}
\end{table}

\paragraph{Results.} \Cref{tab:subcapacity} supports three conclusions.
(i)~\emph{Errors are structured, not noise.} For $d \in \{24, 32, 48\}$, 84--89\% of errors fall into a predicted family, almost all of them gating failures, i.e., \emph{intra-subject} interference between relation blocks.
(ii)~\emph{Outputs are mixtures.} At $d = 48$, the softmax mass splits between the correct attribute (36\%) and the same subject's other attributes (34\%), and on errors the correct attribute has median rank 3 of 4096.
(iii)~\emph{The structure depends on learned geometry.} At matched accuracy, frozen-embedding models make far fewer gating errors (10\% vs.\ 78\% at accuracy ${\approx}\,0.07$; 39\% vs.\ 84\% at ${\approx}\,0.57$), and their subject crowding stays near chance.

\section{Real-LM Probing — Details}\label{app:real-lm-probing}

This appendix backs the experiment in \cref{sec:real-lm-probing}: how the subject--attribute dataset was generated, the probe-training recipe, and the full per-(model, category) results.

\subsection{Data generation pipeline}\label{app:real-lm-probing-data}

We assemble subjects and ground-truth attributes for five categories — \textit{people}, \textit{companies}, \textit{films}, \textit{buildings}, \textit{programming languages} — by prompting Claude, with a confidence-or-skip filter applied at the entity level. 

\paragraph{Generation (\texttt{claude-sonnet-4-6}).} Each category specifies an \texttt{entity\_kind} string and a list of \texttt{(field\_key, field\_description)} pairs, one per relation in that category (e.g.\ for \textit{people}: \texttt{name}, \texttt{city\_of\_birth}, \texttt{country\_of\_birth}, \texttt{father\_name}, \texttt{mother\_name}, \texttt{main\_spouse\_name}, \texttt{gender}, \texttt{religion}, \texttt{occupation}, \texttt{main\_language}, \texttt{century\_of\_birth}). We iterate batches of \texttt{BATCH\_SIZE=50} until the category yields fewer than $\texttt{EARLY\_STOP\_NEW\_UNIQUE}=8$ new uniques in a batch, or hits a target of $\sim$2000 entities. Each batch includes the running exclusion list (capped at 600 entities) so the model does not repeat. The entity is dropped if any attribute matches the placeholder set $\{$``unknown'', ``N/A'', ``varies'', ``uncertain'', ``?'', $\ldots$$\}$.  The verbatim generation prompt (with \texttt{\{...\}} placeholders evaluated at call time):

\begin{small}\begin{verbatim}
You are generating a high-quality factual dataset of {entity_kind}s.

### Task
Produce a JSON array of {batch_size} entities. For each entity, fill in EVERY one of the
following attributes. Each entity is a flat JSON object with exactly these keys:

  - "name": Canonical name
  - "city_of_birth": City of birth
  ...    (all relation-fields for this category)

### Already generated (DO NOT repeat)
You have already generated {N} entities for this category. Here is a representative
sample of {min(N,600)} of them -- DO NOT include any of these or any close variant
of them again:
[ "entity_a", "entity_b", ... ]

### CRITICAL -- Confidence-or-skip rule
Only include an entity if you are CERTAIN about EVERY SINGLE ONE of its attributes.
If you have any doubt about even one field for a given entity, SKIP that entity
entirely and pick a different one. NEVER write "unknown", "N/A", "varies", or any
placeholder. NEVER guess. NEVER infer plausibly. It is far better to return 30
fully-confident entities than 50 with one shaky field each.

### CRITICAL -- Returning fewer than {batch_size} is allowed and expected
{batch_size} is a CEILING, not a quota. If you genuinely cannot find {batch_size}
new confident entities that are not in the exclusion list, return fewer. A short
list signals that the well is dry for this category, so the pipeline should stop.

### CRITICAL -- Well-known entities first; real entities only
Prefer widely documented, canonical entities (Wikipedia-grade). No fictional,
hypothetical, or composite entities. No generic placeholders.

### Output format
YOUR RESPONSE MUST BE A SINGLE JSON ARRAY AND NOTHING ELSE. The very first
character must be [ and the very last must be ]. Each element has exactly the
keys: {field_keys_json}.
\end{verbatim}
\end{small}

\paragraph{Templates.} We use 5 paraphrased prompt templates per (category, relation). Example for \textit{people / century\_of\_birth}:
\begin{small}\begin{verbatim}
"{subj} was born in the {obj}"
"The century in which {subj} was born is the {obj}"
"Historically, the era of {subj}'s birth falls in the {obj}"
"By century, {subj} entered the world in the {obj}"
"On the timeline of centuries, {subj} was born during the {obj}"
\end{verbatim}
\end{small}
One template is sampled deterministically per probe run and used for every entity in that run.

\textbf{Corpus.} The corpus composition (entities and relations per category, after tokenization filtering) is summarized in \cref{tab:lm-recall-corpus}.

\begin{table}[h]
\centering
\small
\renewcommand{\arraystretch}{1.15}
\resizebox{\linewidth}{!}{%
\begin{tabular}{l@{\hspace{1.2em}}r@{\hspace{1.2em}}r@{\hspace{1.2em}}p{0.55\linewidth}}
\toprule
\textbf{Category} & \textbf{\#~entities} & \textbf{\#~relations} & \textbf{Relations} \\
\midrule
\textit{buildings}             & $1{,}662$ & $9$  & architect, architectural\_style, building\_type, completion\_year, construction\_start\_year, floor\_count, height\_meters, location, purpose \\
\textit{companies}             & $1{,}312$ & $6$  & country\_of\_incorporation, founding\_year, headquarters\_location, industry, legal\_structure, main\_founder \\
\textit{films}                 & $\phantom{0}\phantom{0}442$ & $12$ & based\_on, country\_of\_origin, director, distributor, genre, lead\_cast, original\_language, producer, production\_company, release\_year, runtime\_minutes, screenwriter \\
\textit{people}                & $\phantom{0}\phantom{0}866$ & $7$  & century\_of\_birth, city\_of\_birth, country\_of\_birth, gender, main\_language, occupation, religion \\
\textit{programming\_languages} & $\phantom{0}\phantom{0}328$ & $8$  & designer, file\_extension, first\_appeared\_year, implementation\_language, influenced\_by, license, paradigm, typing\_discipline \\
\midrule
\textbf{Total} & $4{,}610$ & $42$ & \\
\bottomrule
\end{tabular}%
}
\caption{Entity and relation pool for the LM-recall corpus. Each (entity, relation) instance is rendered through five paraphrased templates, yielding ${\approx}184$k training sentences in total.}
\label{tab:lm-recall-corpus}
\end{table}

\subsection{Probe training details}\label{app:real-lm-probing-train}

For each (category, relation) pair we instantiate a per-relation \texttt{LinearProbe(d\_in $\to$ rank $\to$ d\_out)} with biases on both factors. \texttt{d\_in} is the dimension of the (aggregated) subject hidden state at the chosen layer; \texttt{d\_out} is the LM's hidden width (= the unembedding row width). We use rank $k=512$ uniformly across models, AdamW with learning rate $10^{-3}$ and weight decay $10^{-3}$, 200 epochs of training with patience-5 early stopping, and the cosine objective $\mathcal{L} = 1 - \cos(W_r \mathbf{s}_\ell + b_r,\, D_{[t_o]})$. The split is subject-disjoint 80/20 within each (category, relation): no subject appears in both train and val. We run two subject-aggregation modes: \texttt{last} (default; numbers in the body and in \cref{tab:real-lm-best-mrr}) and \texttt{mean} (a robustness check, comparable in magnitude). Probes are fit independently per layer, so the \emph{best-layer} numbers report the best per-(model, category, layer) hits@1.

\subsection{Full results}\label{app:real-lm-probing-results}

\cref{tab:real-lm-best-mrr} reports best-layer MRR for every (model, category) we ran, alongside the layer-0 MRR (token-embedding lookup) and the index of the best layer as a fraction of total depth. The pattern is consistent: layer-0 carries some signal, the best layer is in the latter half of the stack, and best-layer MRR is significantly higher than layer-0.

\cref{fig:real-lm-perlayer} shows the layer-by-layer profile for all five LMs on the \textit{people} category. Across every model, MRR / Hits@1 / Hits@10 climb through the network's middle layers and plateau in the upper layers; layer counts differ but the qualitative shape is the same. \cref{fig:real-lm-per-relation} then shows the per-relation breakdown at the best layer for one representative model (Qwen3-14B), revealing which relations are easier (e.g.\ \textit{gender}, \textit{country\_of\_birth}) and which are harder (e.g.\ \textit{religion}, \textit{occupation}).

\begin{figure}[h]
    \centering
    \begin{minipage}[t]{0.48\textwidth}\centering
        \includegraphics[width=\linewidth]{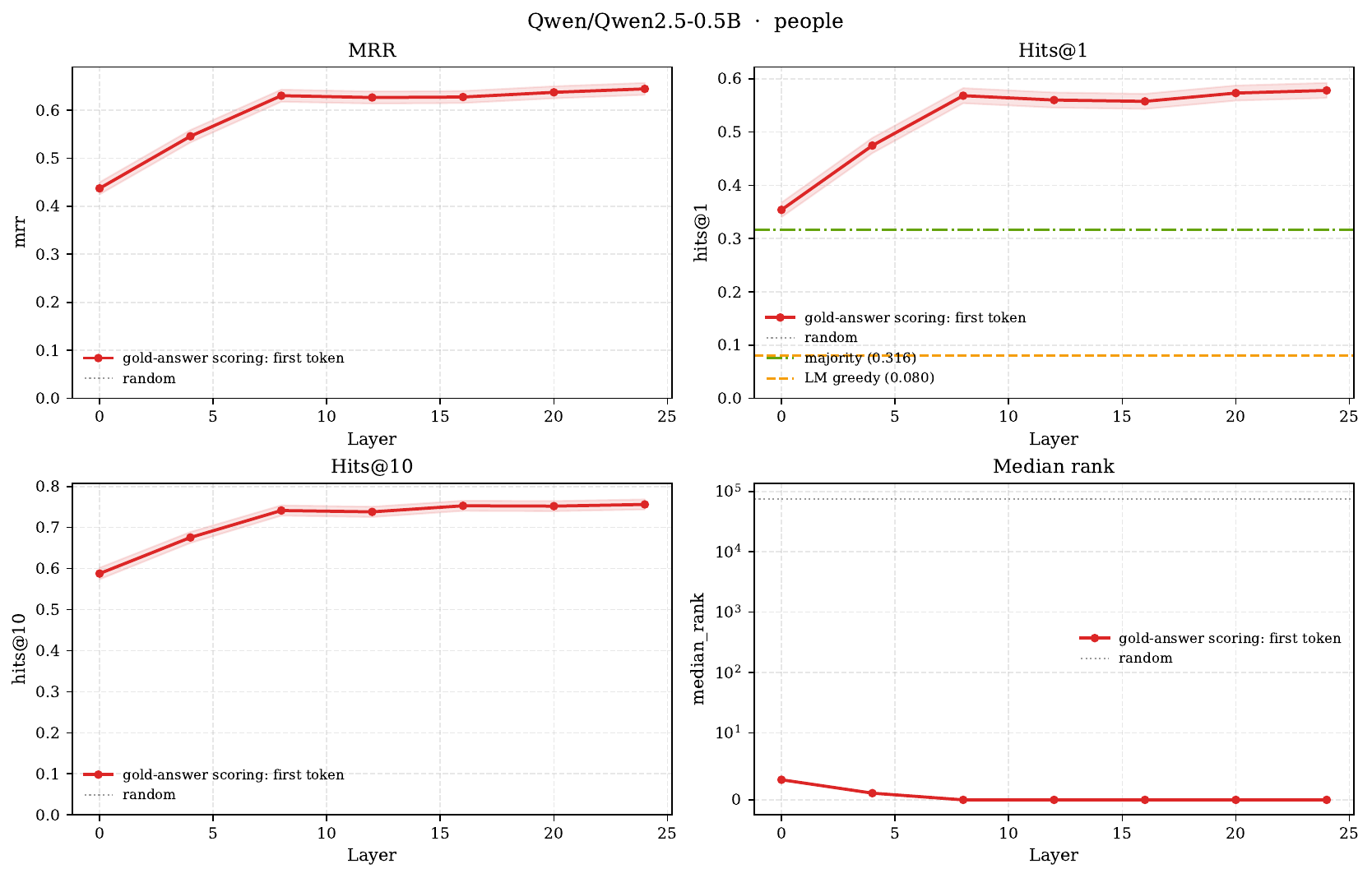}\\
        \textbf{Qwen2.5-0.5B}
    \end{minipage}\hfill
    \begin{minipage}[t]{0.48\textwidth}\centering
        \includegraphics[width=\linewidth]{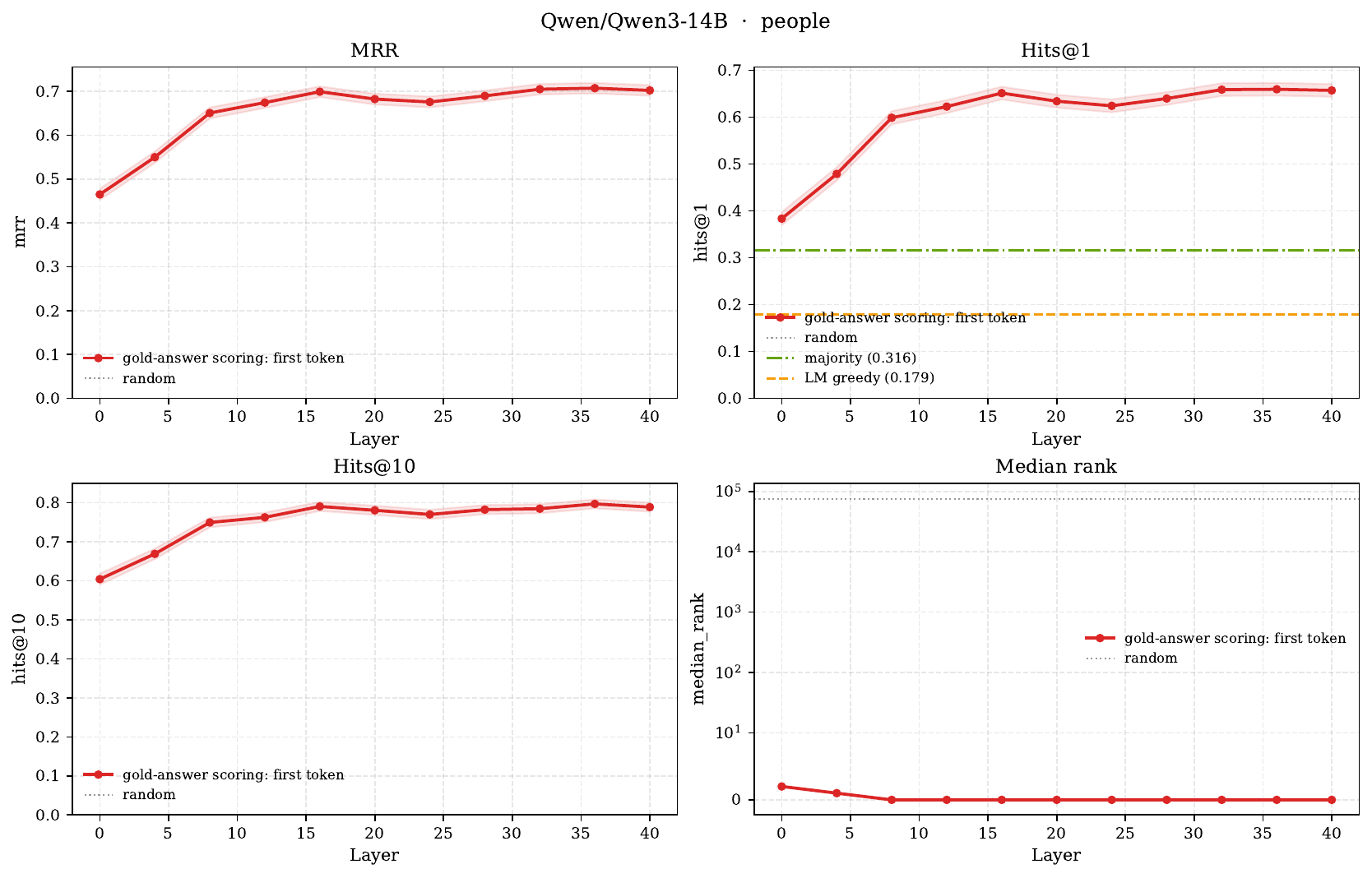}\\
        \textbf{Qwen3-14B}
    \end{minipage}

    \vspace{4pt}

    \begin{minipage}[t]{0.48\textwidth}\centering
        \includegraphics[width=\linewidth]{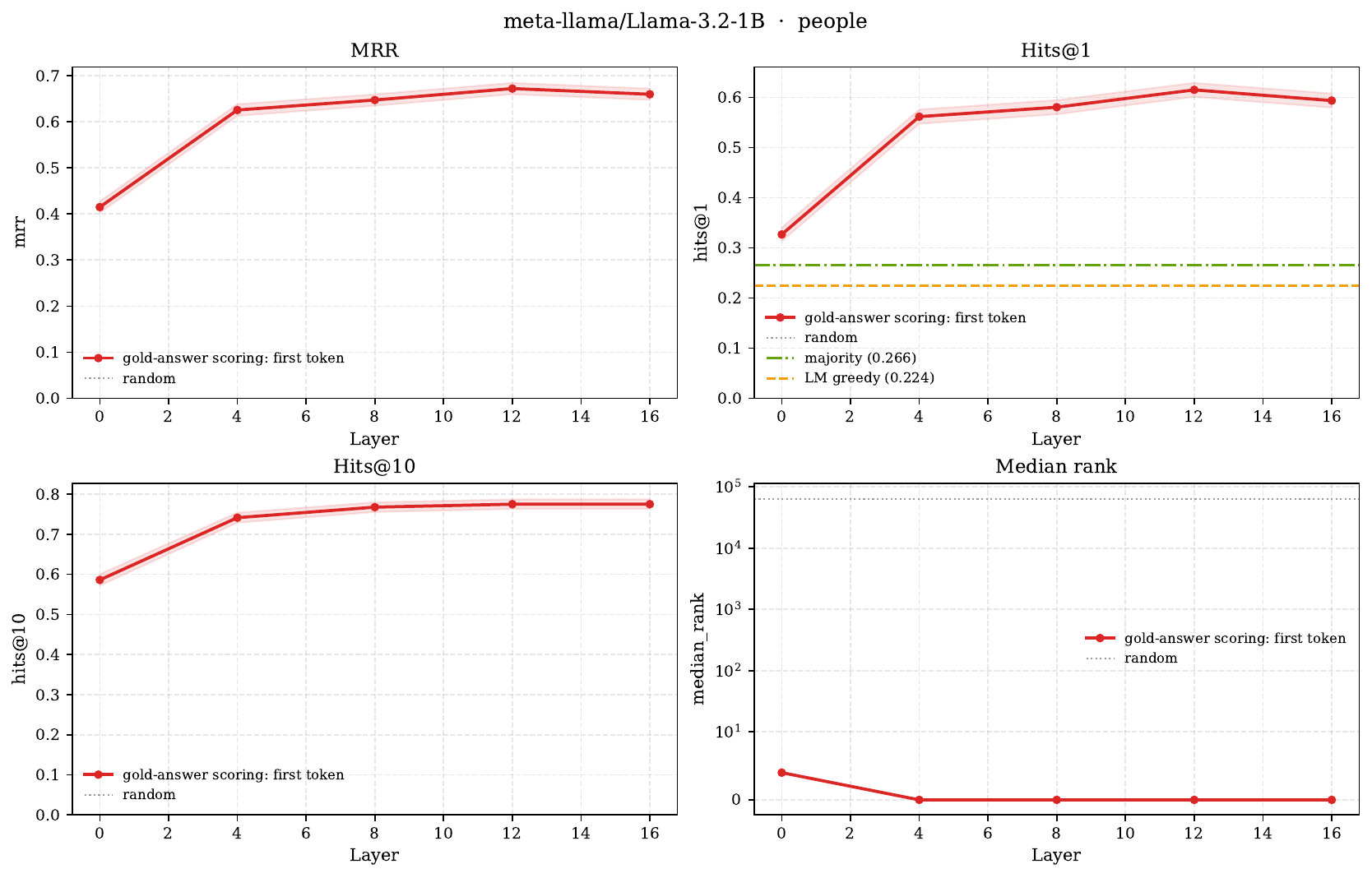}\\
        \textbf{Llama-3.2-1B}
    \end{minipage}\hfill
    \begin{minipage}[t]{0.48\textwidth}\centering
        \includegraphics[width=\linewidth]{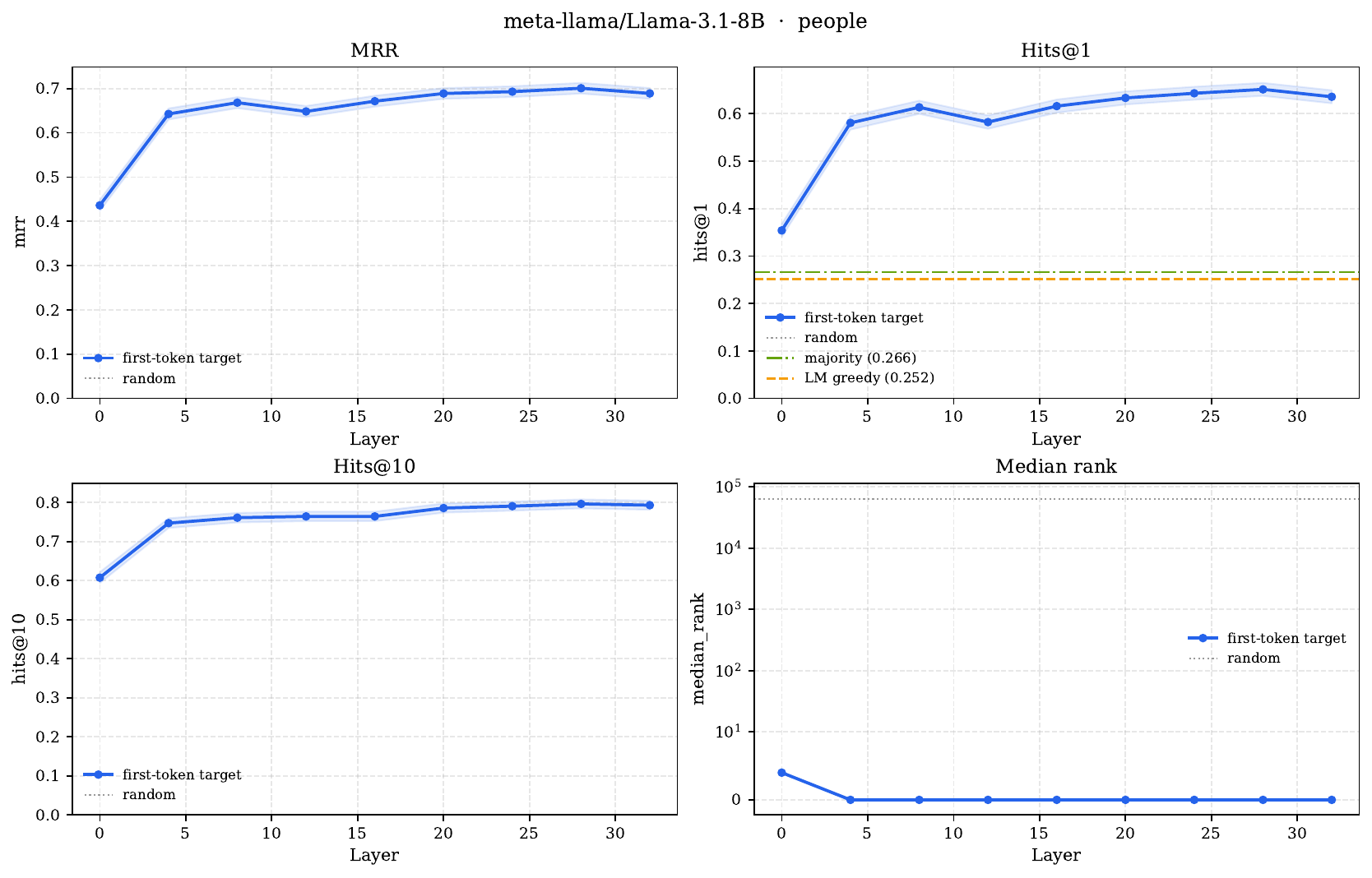}\\
        \textbf{Llama-3.1-8B}
    \end{minipage}

    \vspace{4pt}

    \begin{minipage}[t]{0.48\textwidth}\centering
        \includegraphics[width=\linewidth]{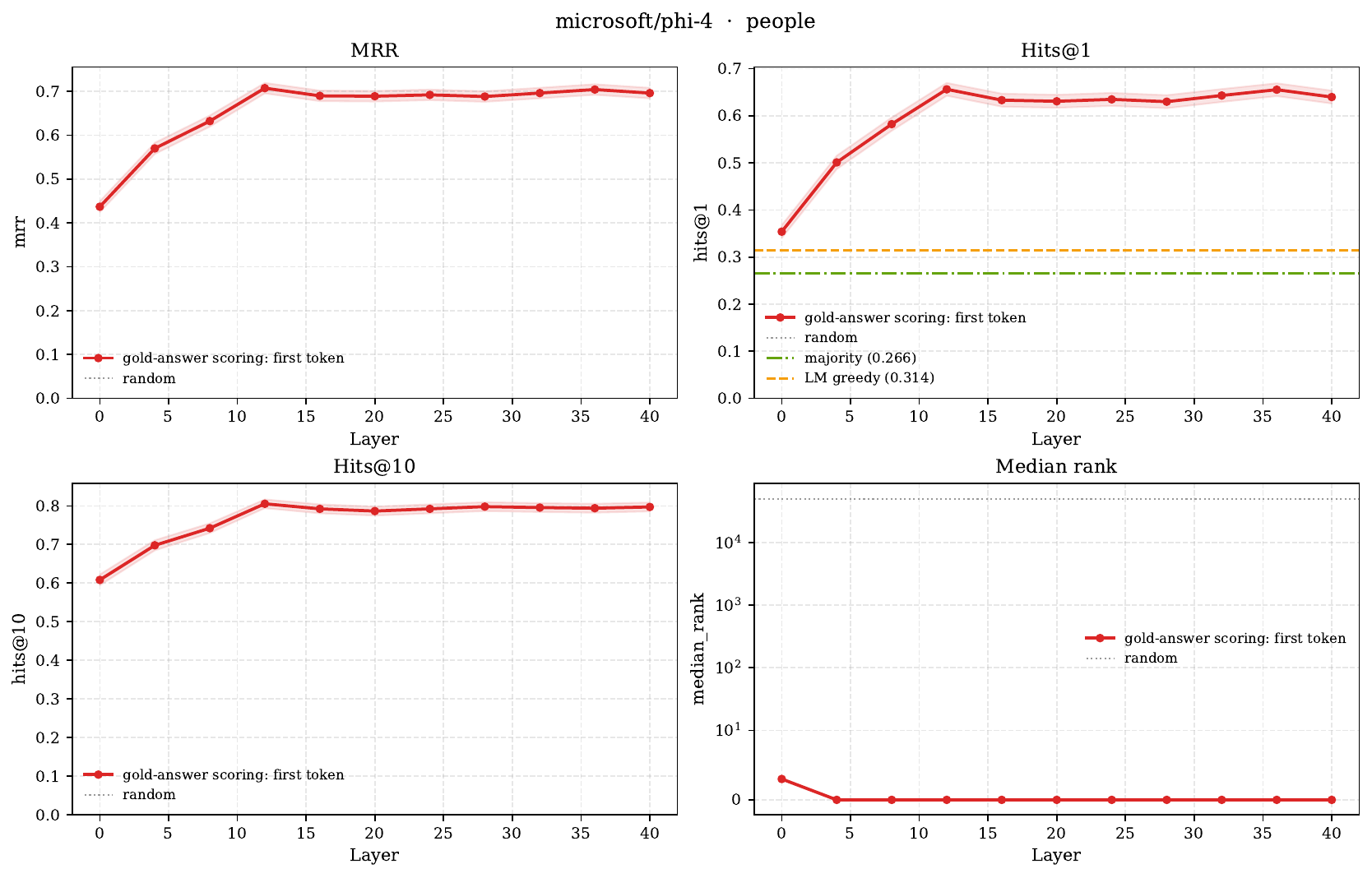}\\
        \textbf{Phi-4}
    \end{minipage}

    \caption{Per-layer linear-readout MRR / Hits@1 / Hits@10 on the \textit{people} category for every LM in the sweep. Across all five models the same qualitative trend holds: layer-0 carries non-trivial signal, the answer's output-embedding direction concentrates in the subject's hidden state through the middle of the network, and the curves plateau in the upper layers (with slight decline at the very top of some models). Layer count varies (24 for Qwen2.5-0.5B, 16 for Llama-3.2-1B, 32 for Llama-3.1-8B, 40 for Qwen3-14B and Phi-4), but the \emph{fraction} of depth at which the peak occurs is roughly comparable.}
    \label{fig:real-lm-perlayer}
\end{figure}

\begin{figure}[h]
    \centering
    \includegraphics[width=0.75\textwidth]{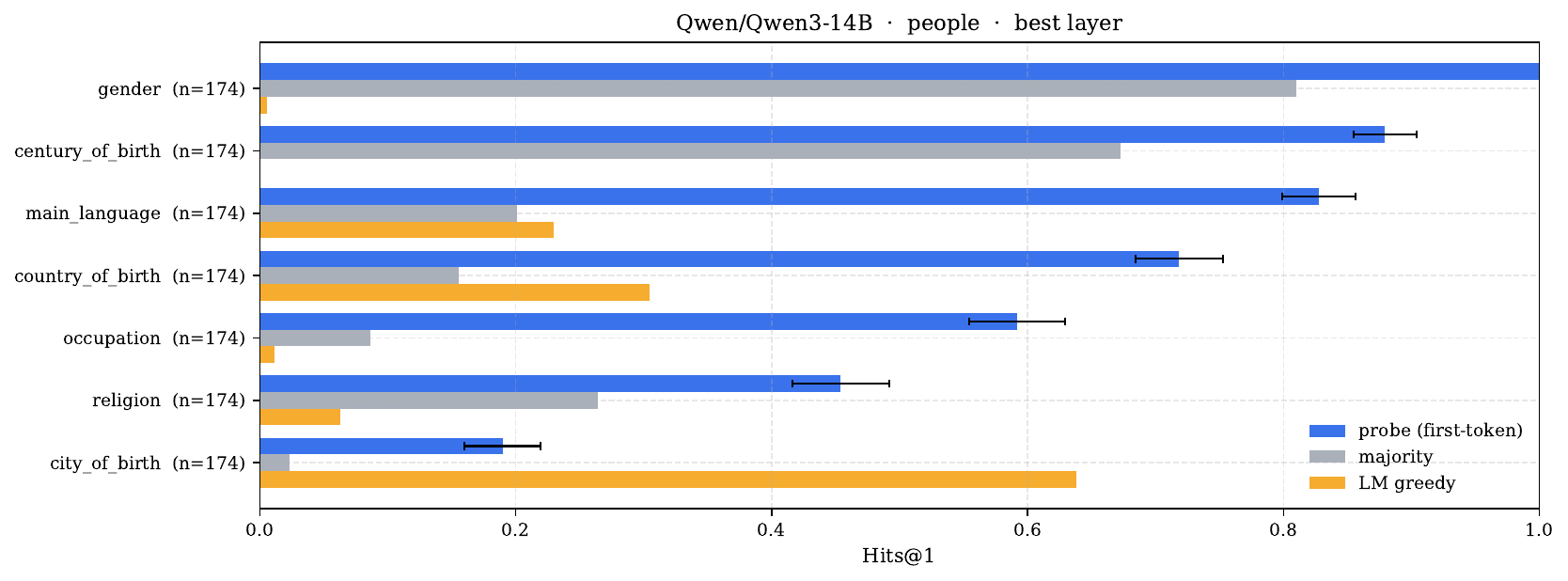}
    \caption{Per-relation Hits@1 at the best layer for Qwen3-14B, \textit{people}. The probe is fit independently per (category, relation), and Hits@1 varies considerably across relations.}
    \label{fig:real-lm-per-relation}
\end{figure}

\begin{table}[h]
\centering
\small
\resizebox{\linewidth}{!}{%
\begin{tabular}{l l c c c c}
\toprule
Model & Category & Layer-0 MRR & Best-layer MRR & Best-layer Hits@1 & Best layer / depth \\
\midrule
Qwen2.5-0.5B & people & 0.437 & 0.644 & 0.578 & 24/24 \\
Qwen2.5-0.5B & companies & 0.440 & 0.565 & 0.523 & 24/24 \\
Qwen2.5-0.5B & buildings & 0.477 & 0.583 & 0.489 & 24/24 \\
Qwen2.5-0.5B & films & 0.315 & 0.381 & 0.326 & 24/24 \\
Qwen2.5-0.5B & prog. langs. & 0.451 & 0.528 & 0.477 & 24/24 \\
\midrule
Qwen3-14B & people & 0.465 & 0.707 & 0.659 & 36/40 \\
Qwen3-14B & companies & 0.461 & 0.640 & 0.594 & 36/40 \\
Qwen3-14B & buildings & 0.490 & 0.672 & 0.588 & 40/40 \\
Qwen3-14B & films & 0.353 & 0.514 & 0.464 & 36/40 \\
Qwen3-14B & prog. langs. & 0.483 & 0.606 & 0.553 & 12/40 \\
\midrule
Llama-3.1-8B & people & 0.436 & 0.701 & 0.651 & 28/32 \\
Llama-3.1-8B & companies & 0.361 & 0.575 & 0.510 & 24/32 \\
Llama-3.1-8B & buildings & 0.308 & 0.547 & 0.462 & 32/32 \\
Llama-3.1-8B & films & 0.256 & 0.460 & 0.404 & 28/32 \\
Llama-3.1-8B & prog. langs. & 0.466 & 0.553 & 0.475 & 8/32 \\
\midrule
Llama-3.2-1B & people & 0.415 & 0.672 & 0.615 & 12/16 \\
Llama-3.2-1B & companies & 0.343 & 0.535 & 0.468 & 16/16 \\
Llama-3.2-1B & buildings & 0.296 & 0.465 & 0.369 & 16/16 \\
Llama-3.2-1B & films & 0.249 & 0.390 & 0.335 & 12/16 \\
Llama-3.2-1B & prog. langs. & 0.430 & 0.516 & 0.453 & 8/16 \\
\midrule
Phi-4 & people & 0.437 & 0.707 & 0.656 & 12/40 \\
Phi-4 & companies & 0.363 & 0.565 & 0.500 & 40/40 \\
Phi-4 & buildings & 0.303 & 0.536 & 0.449 & 40/40 \\
Phi-4 & films & 0.261 & 0.434 & 0.374 & 16/40 \\
Phi-4 & prog. langs. & 0.440 & 0.545 & 0.466 & 12/40 \\
\bottomrule
\end{tabular}
}
\caption{Per-relation linear-readout best-layer results, subject-mean readout target = LM-head row of the answer's first token. Each cell is averaged over the relations within that category. Layer-0 MRR is non-trivial but consistently lower than the best layer, supporting the layer-by-layer enrichment picture in \cref{sec:real-lm-probing}.}
\label{tab:real-lm-best-mrr}
\end{table}

\subsection{Replication on T-REx}\label{app:real-lm-probing-trex}

To verify that our findings are not an artifact of the Claude-compiled dataset, we repeat the probing analysis on T-REx \citep{elsahar2018trex}, a human-curated alignment of Wikipedia text with Wikidata triples. We select T-REx relations of the same kinds as those in our dataset and, following the evaluation convention of \citet{hernandez2023linearity}, keep only facts that each probed LM demonstrably knows (the gold answer's first token is among the model's top-3 next-token predictions). As a result, the evaluation set is model-specific, and numbers are not directly comparable across models. Probes are rank-256 linear maps from the mean subject representation to the output embedding of the answer's first token, trained with the cosine objective of \cref{sec:real-lm-probing} and scored against the full vocabulary. \Cref{tab:trex} reports best-layer Hits@1.

\begin{table}[h]
\centering
\small
\begin{tabular}{llccc}
\toprule
T-REx relation & Counterpart in our data & Llama-3.2-1B & Llama-3.1-8B & Qwen3-14B \\
\midrule
languages spoken & \texttt{main\_language}  & 0.85 (12/16) & 0.86 (16/32) & 0.85 (30/40) \\
country          & --                       & 0.60 (12/16) & 0.62 (24/32) & 0.57 (30/40) \\
citizenship      & --                       & 0.55 (12/16) & 0.57 (24/32) & 0.54 (20/40) \\
genre            & \texttt{genre}           & 0.42 (12/16) & 0.42 (8/32)  & 0.40 (30/40) \\
place of birth   & \texttt{city\_of\_birth} & 0.30 (12/16) & 0.35 (24/32) & 0.34 (40/40) \\
\midrule
\multicolumn{2}{l}{Aggregate over all relations} & \textbf{0.55} & \textbf{0.54} & \textbf{0.53} \\
\multicolumn{2}{l}{\textit{Majority baseline, same set}} & \textit{0.18} & \textit{0.14} & \textit{0.18} \\
\bottomrule
\end{tabular}
\caption{Best-layer Hits@1 of linear probes on T-REx (best layer / depth in parentheses). Evaluation sets are restricted to facts each model knows and are therefore model-specific. ``--'': no direct counterpart in our dataset.}
\label{tab:trex}
\end{table}

Best-layer decodability is 3--4$\times$ above the majority baseline in every model and in the range of \cref{tab:real-lm-best-mrr}. The layer-wise structure also replicates in every model, with depth normalized to make models comparable: aggregate Hits@1 rises from a non-trivial layer-0 value, concentrates through the middle of the network, plateaus, and dips slightly at the final layer, peaking at about 75\% of depth in all three models, as in \cref{fig:real-lm-perlayer}.

\subsection{Training Transformer Language Models on Natural Language Relational Data}

To complement the above experiment, we also train small Transformer LMs on the same natural language relational data used for probing. We instantiate a single-layer GPT-style decoder (Pre-RMSNorm; one attention head; untied input/output embeddings; learned absolute position embeddings; embedding dimension $d{=}256$; MLP hidden width $4d$; $\sim\!7.3\text{M}$ parameters in total) over a word-level vocabulary of size $|V|{=}12{,}752$ built from the union of all rendered prompts. The corpus contains ${\approx}184$k sentences obtained by rendering each of $42$ relations across $5$ of the categories from \cref{sec:real-lm-probing} (\textit{people}, \textit{companies}, \textit{films}, \textit{buildings}, \textit{programming\_languages}) through the same five paraphrased templates per relation. Training uses full-sequence next-token cross-entropy for $20$ epochs (batch size $1024$, AdamW with constant learning rate $5{\times}10^{-4}$ and weight decay $10^{-3}$).

After training, the LM has effectively memorized the corpus, attaining greedy Hits@1 of $99.4\%$ on a $16{,}384$-example subsample (per-category range $[96.9\%, 99.9\%]$). We then apply the linear-readout protocol of \cref{sec:real-lm-probing} (aggregating over \emph{all} the subject's tokens, cosine scoring against the LM-head row of the answer's first token, subject-disjoint $90/10$ split inside each $(\text{category}, \text{relation})$ bucket). The post-embedding hidden state (layer 0) admits substantially sharper linear decoding than the post-block representation (layer 1): mean full-vocabulary Hits@1 of $0.71$ vs.\ $0.52$, with per-category averages of $0.85/0.57$ (\textit{people}), $0.80/0.63$ (\textit{programming\_languages}), $0.73/0.68$ (\textit{companies}), $0.69/0.48$ (\textit{buildings}), and $0.56/0.32$ (\textit{films}). This is consistent with our theoretical construction (\cref{sec:fact_recall_theory}): the subject embedding row itself encodes the answers as a near-linear superposition over the $R$ relations, shaped by gradients backpropagating through attention from the answer position; at the subject position, by contrast, the block's local cross-entropy target is the \emph{next template token}, not the answer, so the post-block representation does not preserve the linear-superposition structure as cleanly.

\paragraph{Influence of tokenization.} In this experiment, each word in the corpus gets a separate token, resulting in multi-token subjects and attributes. Interestingly, in a similar experiment with a tokenizer that allocates a single token for each entity (subject or attribute)---even when the entity is spanned over several distinct words---we see even higher linear readout of the attribute information (over 97\%). We leave the study of the effect of tokenization on linear relational encoding to future work.

\section{Disjoint Attributes}\label{appen:disjoint_attr}

Here we consider the simpler setting of the factual recall problem from \cite{nichani2024understanding}. That is, there is a set of attributes $A$ of size $|A| = R\cdot N$. For each relation $r\in[R]$ there is a function $g_r:[N]\rightarrow A$. The disjoint attribute assumption (Assumption $1$ from \cite{nichani2024understanding}) states that if we define $A_r:= \{g_r(s):~s\in[N]\}$, then $A_r\cap A_{r'} = \emptyset$ for every $r\neq r'$.

\begin{theorem}\label{thm:disjoint-attr}
    In the disjoint attribute case, there is a $1$-layer transformer with embedding dimension $d=6\log(NR)$ that correctly solves the subject-relation-attribute problem.
\end{theorem}

\begin{proof}
    We use \cref{lem:orth vectors pm 1} to get vectors $\bv_1,\dots,\bv_N\in\{\pm 1\}^{6\log(N)}$ and $\bu_1,\dots,\bu_R\in\{\pm 1\}^{6\log(R)}$ as described there. The input embedding is $d = 6\log(N) + 6\log(R)$, and for subject $i\in[N]$ relation $j\in[R]$ and attribute $g_j(i)$ we define their embedding as:
    \begin{equation}
        x_i = \begin{pmatrix}
            \bv_i \\ \pmb{0}_{6\log(R)}
        \end{pmatrix},\quad r_j = \begin{pmatrix}
            \pmb{0}_{6\log(N)} \\ \bu_j
        \end{pmatrix}, \quad g_j(i) = \begin{pmatrix}
            \bv_i \\ \bu_j
        \end{pmatrix}~.
    \end{equation}

    The input to the transformer will be the tokens $x_i$ and $r_j$, and its output (on the second token) will be the vector $g_j(i)$ representing the attribute. 

    In the self-attention layer, we use uniform attention and set $W_V = 2 \cdot I_d$. Now, the output of the self-attention layer will be the sum of the two tokens divided by $2$. After applying the output matrix, we get exactly the given attribute token.
    
\end{proof}

\section{Proofs from \Cref{sec:fact_recall_theory}}
In the proofs, we will use the following lemma:

\begin{restatable}{lemma}{lemOrthVectors}\label{lem:orth vectors pm 1}
    For any $k \geq 2$ there exist $\bv_1,\dots,\bv_k\in\{\pm1\}^{6\log(k)}$ such that $|\inner{\bv_i,\bv_j}| \leq \frac{9}{2}\log(k)$ for all $i\neq j$, where $\log$ is base $2$.
\end{restatable}

\begin{proof}
    We sample $\bv_1,\dots,\bv_k\in\{\pm 1\}^{6\log(k)}$ uniformly at random. Note that for any $i \neq j$ we have that $\mathbb{E}[\inner{\bv_i,\bv_j}] = 0$. By Hoeffding's inequality we have that:
    \begin{equation}
        \pr\left(\left|\inner{\bv_i,\bv_j} \right| \geq \tfrac{9}{2}\log(k) \right) \leq 2\exp\left(-\frac{(9\log(k)/2)^2}{2\cdot 6\log(k)}\right) = 2\exp\left(-\tfrac{27}{16}\log(k)\right) \leq 2k^{-12/5}~.
    \end{equation}
    Applying union bound on the above for all $k(k-1)/2$ pairs $i \neq j$ we get that:
    \[
    \pr\left(\forall i\neq j,~\left|\inner{\bv_i,\bv_j} \right| \leq \tfrac{9}{2}\log(k) \right) \geq 1 - 2k^{-12/5}\cdot \frac{k^2}{2} = 1 - k^{-2/5} > 0~.
    \]

In particular, there exist such vectors $\bv_1,\dots,\bv_k$.
\end{proof}

\subsection{Proof of \Cref{thm:shared attr}}\label{appen:proof_shated_attr}

\begin{proof}
We begin with the MLP construction, and later provide the multi-head attention construction.
\paragraph{MLP construction (with uniform attention).}
    We use \cref{lem:orth vectors pm 1} to get vectors $\bv_1,\dots,\bv_N\in\{\pm 1\}^{6\log(N)}$ as described there. The input embedding is $d = 6R\log(N) + 1$, and for every subject $i\in[N]$ relation $j\in[R]$ and attribute $g_j(i)$ we define their embedding as:
        \begin{equation}\label{eq:embedding x r g}
        x_i = \begin{pmatrix}
            \bv_{g_1(i)} \\ \vdots \\ \bv_{g_R(i)} \\ 0
        \end{pmatrix},\quad r_j = \begin{pmatrix}
            \pmb{0}_{6R\log(N)} \\ j
        \end{pmatrix}, \quad g_j(i) = \begin{pmatrix}
            \bv_{g_j(i)} \\ \pmb{0} 
        \end{pmatrix}~.
    \end{equation}

    The input to the transformer will be the tokens $x_i$ and $r_j$, and its output (on the second token) will be the vector $g_j(i)$. For simplicity, we will output its first $6\log(N)$ coordinates representing the correct vector $\bv_{g_j(i)}$, and ignore the rest of this vector, which are just zeroes (we keep it this way to be consistent with having the same embedding dimension for subjects and attributes). We will refer to the output as just the vector $g_j(i)$.

     In the self-attention layer, we use uniform attention and set $W_V = 2 \cdot I_d$. Now, the output of the self-attention layer will be the sum of the two tokens divided by $2$. After applying the output matrix, we get the vector $
         \begin{pmatrix}
            \bv_{g_1(i)} \\ \vdots \\ \bv_{g_R(i)} \\ j
     \end{pmatrix}$. We use the ReLU MLP on this vector to apply the following operation:
     \begin{align*}
        f_j(z) = 2\left([-z+j]_+ + [z-j]_+ \right)~.    
     \end{align*}
     Note that for $z\in\naturals$ we have that $f_j(z) = 0$ if $z=j$, while $f_j(z) \geq 2$ for any $z\neq j$.
     We construct an MLP that performs the following:
     \begin{equation}
         h\left(\begin{pmatrix}
            \bv_{g_1(i)} \\ \vdots \\ \bv_{g_R(i)} \\ j
     \end{pmatrix}\right) = \sum_{\ell = 1}^R [\bv_{g_\ell(i)} - f_{\ell}(j)\pmb{1}_{6\log(N)}]_+ - [-\bv_{g_\ell(i)} - f_{\ell}(j)\pmb{1}_{6\log(N)}]_+~.
     \end{equation}
     This is a 3-layer MLP with $O(R\log(N))$ neurons. Note that the $\ell = j$ term outputs exactly $\bv_{g_j(i)}$, since $f_j(j) = 0$. For any other $\ell\neq j$, we have $f_\ell(j) \geq 2$, and since each coordinate of $\bv_{g_\ell(i)}$ is $\pm 1$, the terms inside both ReLUs are negative, so the $\ell$-th term vanishes. Hence $h$ outputs exactly $\bv_{g_j(i)}$.

\paragraph{Multi-head attention construction (with no MLP).}

We now provide a different construction where the non-linear selection mechanism is implemented by attention instead of an MLP.
Define the embeddings
\begin{equation}\label{eq:attention_embeddings}
        x_i = \begin{pmatrix}
            \bv_{g_1(i)} \\ \vdots \\ \bv_{g_R(i)} \\ \pmb{0}_{6 \log(R)} \\ 1
        \end{pmatrix},\quad r_j = \begin{pmatrix}
            \pmb{0}_{6R\log(N)} \\ \bu_{j} \\ -1
        \end{pmatrix}, \quad g_j(i) = \begin{pmatrix}
            \bv_{g_j(i)} \\ \pmb{0} 
        \end{pmatrix}~,
\end{equation}
where we additionally consider vectors $\bu_1, \ldots, \bu_R \in \{\pm 1\}^{d_u}$ with~$d_u = 6 \log(R)$ from~\cref{lem:orth vectors pm 1}.
Then we consider~$R$ attention heads, one for each relation, with the following matrices, for a large~$\beta > 0$:
\begin{align}
    W_K^j &= \begin{pmatrix}
            \pmb{0} \\ \beta
        \end{pmatrix} \in \R^{d \times 1} \\
    W_Q^j &= \begin{pmatrix}
        \pmb{0}_{6R \log(N)} \\
        2 \bu_j \\
        \frac{7}{4} d_u
    \end{pmatrix} \in \R^{d \times 1} \\
    W_V^j &= \begin{pmatrix}
        \pmb{0} \\
        \vdots \\
        \pmb{0} \\
        I_{6 \log(N)} \\
        \pmb{0} \\
        \vdots \\
        \pmb{0}
    \end{pmatrix} \in \R^{d \times 6 \log(N)} \\
    W_O^j &= \begin{pmatrix}
        I_{6 \log(N)} \\
        \pmb{0}
    \end{pmatrix} \in \R^{d \times 6 \log(N)}
\end{align}
where the identity block in the value matrix~$W_V^j$ appears in block~$j$ and extracts the correct attribute vector~$\bv_{g_j(i)}$ from~$x_i$.
For the key-query mechanism, we have
\[
W_Q^j r_{j'} = 2 \bu_j^\top \bu_{j'} - \tfrac{7}{4} d_u, \text{ which equals } \tfrac{1}{4} d_u > 0 \text{ if } j = j' \text{ and is at most } -\tfrac{1}{4} d_u < 0 \text{ otherwise},
\]
so that for large enough~$\beta$, the relation token attends to the subject token when the relation is~$j$, and to itself otherwise.
Note that the head dimension for value/output matrices is~$6 \log(N)$, and key/query matrices can be made even smaller (here with internal dimension one, though these can be padded to match the value/output internal dimension).

\end{proof}

\textbf{Hierarchical attention construction.} We can present a hierarchical solution with only 2 attention heads per layer and $1 + \log(R)$ layers. 

We assume each sequence starts with a BOS token $\bb_0$. Define the embeddings
\begin{equation}\label{eq:attention_embeddings}
    \bb_0 = \begin{pmatrix}
                \pmb{0}_{6R\log(N) + \log{R}} \\ 0 \\ 1
        \end{pmatrix}
        \quad
        x_i = \begin{pmatrix}
            \bv_{g_1(i)} \\ \vdots \\ \bv_{g_R(i)} \\ \pmb{0}_{ \log(R)} \\ 1 \\ 0
        \end{pmatrix},\quad r_j = \begin{pmatrix}
            \pmb{0}_{6R\log(N)} \\ \bb(j) \\ -1 \\ 0
        \end{pmatrix}, \quad g_j(i) = \begin{pmatrix}
            \bv_{g_j(i)} \\ \pmb{0} 
        \end{pmatrix}~,
\end{equation}

Where $\bb(j) \in \{\pm 1\}^{\text{log} R}$ (the \emph{bit block}) is the binary encoding of $j$. We will denote layer $l$ attention heads by $A$ and $B$. We call the components of the subject's representation that encode the different answers the \emph{answers stack}. In the first layer, the attention heads copy the information from the subject token into the relation token. That is, we define:

\[
W_Q^{(0)} = \mathbf{0}, \qquad W_K^{(0)} = \mathbf{0}, \qquad W_{VO}^{(0)} = \begin{pmatrix} 3 \cdot I_{6R\log N} & \mathbf{0} \\ \mathbf{0} & \mathbf{0} \end{pmatrix}.
\]

With $W_Q = W_K = 0$, attention is uniform over the causal prefix at each position. At the relation token (position $2$), this places weight $\tfrac{1}{3}$ on each of the three tokens. Of these, only $x_i$ has a nonzero answer-stack region, so the factor of $3$ in $W_{VO}^{(0)}$ compensates for the averaging. The relation token's post-Layer-$0$ residual is
\[
h^{(0)}_r \;=\; r_j + W_{VO}^{(0)} \cdot \tfrac{1}{3}(b_0 + x_i + r_j) \;=\; \begin{pmatrix} \mathbf{v}_{g_1(i)} \\ \vdots \\ \mathbf{v}_{g_R(i)} \\ \mathbf{b}(j) \\ -1 \\ 0 \end{pmatrix}.
\]
The relation's residual now carries the full answer-stack and the bit-block. By causal masking, BOS attends only to itself and its residual is unchanged: $h^{(0)}_b = b_0$. And, the subject's representation will be:

\[
h^{(0)}_\text{subj} = \begin{pmatrix}
            {(\frac{5}{2})}\bv_{g_1(i)} \\ \vdots \\ {(\frac{5}{2})} \bv_{g_R(i)} \\ \pmb{0}_{ \log(R)} \\ 1 \\ 0
        \end{pmatrix}
\]

Finally, the BOS token remains as in the input layer.

\textbf{for layer $\ell$}, we will design two attention heads, $A$ and $B$, such that on each sequence, one of them attends to the relation token, and one of them attends to the BOS token. $W_K^{(\ell)} \in \mathbb{R}^{2 \times d}$ reads the flag and BOS-flag, producing $2$-dimensional keys:

\[
W_K^{(\ell)} = \beta \begin{pmatrix} \mathbf{0}_{d-2}^\top & 1 & 0 \\ \mathbf{0}_{d-2}^\top & 0 & 1 \end{pmatrix}.
\]

The three keys are
\[
k_{\text{subj}} = \beta \cdot (+1, 0), \qquad k_{\text{rel}} = \beta \cdot (-1, 0), \qquad k_{\text{BOS}} = \beta \cdot (0, +1).
\]

The queries are linear functions of the residual's bit-block, designed so that the relation token's query targets the relation token itself when the head's polarity matches bit $\ell$ of $j$, and targets BOS otherwise. Define $W_Q^{(\ell, A)}, W_Q^{(\ell, B)} \in \mathbb{R}^{2 \times d}$ to compute, on a token with bit-block $\mathbf{c}$ ($\mathbf{c}$ is $b_j$ for the relation token, and 0 otherwise):\footnote{the bias term can be implemented by adding a constant coordinate to the vectors.}
\[
q_A = \tfrac{\beta}{2}(-1, 1) + \tfrac{\beta}{2} \mathbf{c}_\ell \cdot (-1, -1), \qquad q_B = \tfrac{\beta}{2}(-1, 1) - \tfrac{\beta}{2} \mathbf{c}_\ell \cdot (-1, -1).
\]
(Head $B$ has the opposite sign on the bit-dependent term.) This was constructed such that head $A$ aligns with $k_{\text{subj}}$ when $\bb(j)_\ell=1$, and with $k_{\text{BOS}}$ when $\bb(j)_\ell=-1$ (and vice versa for $B$). We have:
\[
\begin{array}{c|cc}
\mathbf{b}(j)_\ell & q_A & q_B \\ \hline
+1 & \beta(-1, 0) & \beta(0, +1) \\
-1 & \beta(0, +1) & \beta(-1, 0)
\end{array}
\]

For $\mathbf{b}(j)_\ell = +1$:
\[
\begin{array}{c|ccc|c}
 & q \cdot k_{\text{subj}} & q \cdot k_{\text{rel}} & q \cdot k_{\text{BOS}} & \text{argmax} \\ \hline
\text{Head } A & -\beta^2 & +\beta^2 & 0 & \text{relation (active)} \\
\text{Head } B & 0 & 0 & +\beta^2 & \text{BOS (silent)}
\end{array}
\]

And vice versa for the case $\mathbf{b}(j)_\ell = -1$. Define the index sets $S_\ell^+ = \{r \in [R] : r_\ell = +1\}$ and $S_\ell^- = \{r : r_\ell = -1\}$. Let $P_\ell^+, P_\ell^- \in \mathbb{R}^{d \times d}$ be the orthogonal projectors onto the answer-stack coordinates of slots in $S_\ell^+$ and $S_\ell^-$ respectively (block-diagonal: identity on the relevant slot blocks within the answer-stack region, zero everywhere else).
Then, we set:

\[
W_{VO}^{(\ell, A)} = -P_\ell^-, \qquad W_{VO}^{(\ell, B)} = -P_\ell^+.
\]

By causal masking, BOS attends only to itself in every layer. Since $b_0$'s answer-stack is zero, every $W_{VO}$ applied to $b_0$ produces zero, so $h^{(\ell)}_b = b_0$ for all $\ell$. BOS's value is therefore zero in the answer-stack region throughout, and the silent head's contribution to any other token's residual is zero.

\begin{lemma}
    After layer $\ell$, the hidden state over the relation token $h^{(\ell)}_r$ has in its answer stack (the first $6R\log(N)$ bits) $\mathbf{v}_{g_r(i)}$ in the slot at index $r$ if $r$ agrees with $j$ on bits $1, \ldots, \ell$, and $\mathbf{0}$ otherwise. The bit-block, flag, and BOS-flag are unchanged from $h^{(0)}_r$.
\end{lemma}

\begin{proof}
    We prove by induction on $\ell$. Assume the claim holds for $\ell' < \ell$. For layer $\ell$, without loss of generality, assume $\mathbf{b}(j)_\ell = +1$, so head $A$ fires. The contribution of head $A$ is $W_{VO}^{(\ell, A)} h^{(\ell-1)}_r= -P_\ell^- h^{(\ell-1)}_r$. By the induction hypothesis, $h^{(\ell-1)}_r$ contains in the answer stack all indices that agree with $j$ on bits $1, \dots, \ell-1$ and $0$ in the rest of the indices. Then, after the residual we get

    \[
    h^{(\ell)}_r = h^{(\ell-1)}_r + (-P_\ell^- h^{(\ell-1)}_r) =(I - P_\ell^-) h^{(\ell-1)}_r~,
    \]
    where the surviving slots are those agreeing with $j$ on bits $1, \dots, \ell$. Head $B$ does not contribute since the answer slot of the BOS representation is zero. By construction, $W_{VO}^{(\ell, A)}$ and $W_{VO}^{(\ell, B)}$ have zero effect on the non-answer-stack coordinates, so the bit block, flag and BOS-flag are unchanged.
\end{proof}

After $\log(R)$ filter layers, we end with $h^{(\log R)}_r$, whose answer stack has a single surviving block: the one corresponding to the relation $j$ in the example, containing $\mathbf{v}_{g_j(i)}$. Since all other blocks are zero, a linear map that sums the $R$ blocks of the answer stack outputs exactly $\mathbf{v}_{g_j(i)}$.

\subsection{Proof of \cref{thm:k-hop lower bound}}\label{appen:proof_k_hop_lower_bound}

\begin{proof}
Let $\mathcal{T}$ be a transformer whose self-attention and MLP parameters take $W$ bits in total, and where the embedding of each token takes $D$ bits. Thus, the subject embedding matrix takes $N\cdot D$ bits and the relation embedding matrix takes $R\cdot D$ bits. For simplicity, we assume the unembedding matrix is tied to the embedding matrix, and treat the untied case at the end of each part. We establish the lower bound on $W$ by deriving two independent combinatorial capacity constraints: the global storage bound ($\mathcal{B}_{global}$) and the local decoding bound ($\mathcal{B}_{local}$).

\paragraph{Part I: The Global Storage Bound ($\mathcal{B}_{global}$)}
Let $\mathcal{G} = (V, E)$ be a directed, edge-colored multigraph representing the operation of $R$ distinct relations on $N$ subjects. We define the vertex set $V = [N]$, where each node corresponds to a subject. Each relation $i \in [R]$ operates as a permutation $g_i \in S_N$ on the subjects. Thus, we define the edge set as:$$E = \{ (u, g_i(u), i) \mid u \in V, \, i \in [R] \}$$where each tuple $(u, v, i)$ denotes a directed edge from $u$ to $v$ labeled with relation color $i$. Note this is an $R$-regular graph.
Let $(S_N)^R$ be the space of all possible $R$-regular directed edge-colored graphs on $N$ vertices. The cardinality of this space is $(N!)^R$. Define the evaluation mapping $\Phi: (S_N)^R \to [N]^{N \times R^k}$, which maps a graph $\mathcal{G}$ to its exact $k$-hop transition matrix $\mathbf{O}_k$.

This map is not injective. Let $\mathcal{G} = (g_1, \dots, g_R)$ and $\mathcal{H} = (h_1, \dots, h_R)$ with $\Phi(\mathcal{G}) = \Phi(\mathcal{H})$, so all $k$-hop paths evaluate identically. Fix a suffix word $u \in [R]^{k-1}$. For every $r \in [R]$, the path consisting of $r$ followed by $u$ gives $g_r \circ g_u = h_r \circ h_u$, hence
\[
h_r = g_r \circ C_u, \qquad \text{where } C_u := g_u \circ h_u^{-1} \in S_N .
\]
Note that $C_u$ does not depend on $r$. It also does not depend on $u$: for two suffixes $u, u'$ we get $g_r \circ C_u = h_r = g_r \circ C_{u'}$, and since $g_r$ is a bijection, $C_u = C_{u'}$. Hence there is a single permutation $C \in S_N$ such that $h_r = g_r \circ C$ for all $r \in [R]$.

Thus, $\mathcal{H}$ is a right-translation of $\mathcal{G}$ by a constant permutation $C$, and $\mathcal{H}$ is determined by $\mathcal{G}$ and $C$. Since there are at most $N!$ choices for $C$, the preimage of any output matrix $\mathbf{O}$ under $\Phi$ satisfies $|\Phi^{-1}(\mathbf{O})| \le N!$.

Consequently, the image space of valid target matrices has cardinality at least:
\[
|\text{Im}(\Phi)| \ge \frac{(N!)^R}{N!} = (N!)^{R-1}
\]

The parameters of $\mathcal{T}$ that may depend on the graph are the internal weights ($W$ bits), the subject embedding matrix ($N \cdot D$ bits) and the relation embedding matrix ($R \cdot D$ bits). Hence, $\mathcal{T}$ can implement at most $2^{W + N \cdot D + R \cdot D}$ distinct input-output maps. Two graphs with different $k$-hop transition matrices require different input-output maps, so by the Pigeonhole Principle:
\[
2^{W + N \cdot D + R \cdot D} \ge (N!)^{R-1}
\]

Applying Stirling's approximation $\log_2(N!) \ge N \log_2(N/e)$ and isolating $W$:
\[
W \ge (R-1)N \log_2\left(\frac{N}{e}\right) - N \cdot D - R \cdot D = \mathcal{B}_{global} - R \cdot D
\]

If the unembedding matrix is not tied to the embedding matrix, it adds another $N\cdot D$ bits, and the same argument gives $2^{W + 2N\cdot D + R\cdot D} \ge (N!)^{R-1}$, that is, $W \ge (R-1)N \log_2(N/e) - 2N\cdot D - R\cdot D$. This has the same asymptotic form as $\mathcal{B}_{global}$.

\paragraph{Part II: The Local Decoding Bound ($\mathcal{B}_{local}$)}

\textbf{Extracting disjoint trees.} We pick roots greedily. Let $U = [N]$ be the set of available nodes. Pick a root $v \in U$, and let $T_v$ be the set of nodes reachable from $v$ by at most $k$ hops. Then remove from $U$ every node from which $T_v$ can be reached by at most $k$ hops, that is, all nodes obtained from $T_v$ by applying at most $k$ inverse maps $g_{r}^{-1}$. This guarantees that for any root $v'$ picked later, $T_{v'} \cap T_v = \emptyset$, since otherwise $v'$ would reach $T_v$ within $k$ hops and would have been removed. The set of removed nodes has at most $R^{2k}$ elements. This is only an upper bound: if paths collide or form cycles, fewer nodes are removed. Repeating this until $N/2$ nodes are removed, we obtain at least $p = \lfloor \frac{N}{2R^{2k}} \rfloor$ roots $v_1, \dots, v_p$ with pairwise disjoint trees $T_{v_1}, \dots, T_{v_p}$.

\textbf{Counting the leaves.} For each root we only consider the $M := R^k$ queries $(v_i, w)$ with $w \in [R]^k$, that is, the leaves of a complete $R$-ary tree of depth $k$. Note that the removed neighborhood has size $R^{2k}$ while we count only $R^k$ leaves per root; this is why $R^{2k}$ appears in the denominator of $p$ while $R^k$ appears in $\mathcal{B}_{local}$.

Since the $p$ trees are disjoint, we can assign target vertices to the $p \cdot M$ edges independently. For the $j$-th assignment, there are at least $(N - j)$ available vertices. Because $p \cdot M \le N/2$, the number of valid assignments for these subgraphs is bounded by:
\[
 \prod_{j=1}^{pM} (N - j) \ge \left( \frac{N}{2} \right)^{pM}
\]

These $pM$ queries involve only the embeddings of the $p$ roots and of the $R$ relations, so the part of the embedding matrix that matters has $p\cdot D + R\cdot D$ bits. By the same logic as in the previous case:
\[
2^{W + p \cdot D + R \cdot D} \ge \left( \frac{N}{2} \right)^{pM}
\]

Taking log on both sides yields:
\[
W + p \cdot D + R \cdot D \ge pM(\log_2 N - 1)
\]
\[
W \ge p \Big[ M(\log_2 N - 1) - D \Big] - R \cdot D = \mathcal{B}_{local} - R \cdot D
\]

If the unembedding matrix is not tied to the embedding matrix, the same argument applies, replacing $ p\cdot D$ with $ 2p\cdot D$ and $ D$ with $2D$.

\paragraph{Conclusion}
To achieve zero error on the $k$-hop task across the entire domain $(S_N)^R$, the internal weights $W$ must simultaneously satisfy both the global storage constraint and the local decoding constraint. Therefore:
\[
W \ge \max \Big\{ \mathcal{B}_{global}, \; \mathcal{B}_{local} \Big\} - R \cdot D
\]
\end{proof}

\subsection{Proof of \cref{thm:k-hop-non-cot}}\label{appen:proof_k_hop_no_cot}

\begin{proof}
The input to the transformer in all constructions is $s,r_1,\dots,r_k$. We denote the initial subject as $s_0:= s$, and for step $i\in[k]$ we denote $s_i = g_{r_i}(s_{i-1})$. We also denote $d_S = \lceil \log_2 N \rceil$ and $d_R = \lceil \log_2 R \rceil$.
We prove each construction separately:

\begin{enumerate}

    \item 
The embedding dimension is $D = 1 + d_S + 2d_R + k$. The input sequence consists of $k+1$ tokens that are initialized as:

\begin{align*}
\text{Token } 0 \text{ (Subject):} \quad & x_0^{(0)} = \begin{bmatrix} 1 \\ \text{bin}(s_0) \\ \bm{0}_{2d_R} \\ e_1 \end{bmatrix} \\
\text{Tokens } i \in [1, k] \text{ (Relations):} \quad & x_i^{(0)} = \begin{bmatrix} \bm{0}_{d_S + 1} \\ \text{bin}(r_i) \\ \bm{0}_{d_R} \\ e_i\end{bmatrix}~.
\end{align*}

Here, $\text{bin}(s)\in\{0,1\}^{d_S}$ is a vector representing the binary representation of $s$, and similarly for $\text{bin}(r_i)\in\{0,1\}^{d_R}$. The idea is that the only token that will change over the computation of the transformer is $x_0$. This token will extract the relations one-by-one according to the order dictated by the last $k$ coordinates, and use the MLP to output the next subject.
The self-attention matrices are defined as:
\[
W_K = W_Q =  \begin{bmatrix} 0 & 0 & 0 & 0 \\ 0 & 0 & 0 & 0 \\ 0 & 0 & 0 & 0 \\ 0 & 0 & 0 & I_{k+2} \end{bmatrix}, \quad
W_V = \begin{bmatrix} 0 & 0 & 0 & 0 \\ 0 & 0 & 0 & 0 \\ 0 & I_{d_R} & 0 & 0 \\ 0 & 0 & 0 & 0 \end{bmatrix}~.
\]
The output of the tokens after the self-attention layer in the $\ell$-th step is:
\[
\tilde{x}_0^{(l)} = \begin{bmatrix} 1 \\ \text{bin}(s_{\ell-1}) \\ \bm{0}_{d_R} \\ \text{bin}(r_\ell) \\ 2e_\ell\end{bmatrix}, \quad \text{and} \quad 
\tilde{x}_i^{(l)} = \begin{bmatrix}  \bm{0}_{d_S + 1}  \\ \text{bin}(r_i) \\ \text{bin}(r_\ell) \\ 2e_i\end{bmatrix} \quad \forall i \ge \ell
\]
We now describe the MLP. Its first layer has a width of $O(N\cdot R)$, and we index the neurons by all the possible pairs of subjects and relations $(s,r)\in [N]\times [R]$. The weight vector and bias of each neuron are:
\[
(W)_{s,r} = \begin{bmatrix}
    1 \\ 2\text{bin}(s) - \bm{1}_{d_S} \\ \bm{0}_{d_R} \\ 2\text{bin}(r) - \bm{1}_{d_R}  \\ \bm{0}_k
\end{bmatrix},\quad (b)_{s,r} = -\norm{\text{bin}(s)}_1 - \norm{\text{bin}(r)}_1~.
\]
Also, the input vector is passed to the next layer unchanged. This can be done by implementing the identity using ReLU: $z = \text{ReLU}(z) - \text{ReLU}(-z)$. Note that the output of this layer on these $N\cdot R$ coordinates is a one-hot vector representing the exact subject and relation that appear in the input token. Also, for every token where the subject is the zero vector, the output is also the zero vector. The next layer of the MLP is used as a projection over these $N\cdot R$ dimensions, which outputs the subject $g_r(s)$. Finally, the last layer will zero out the $d_R$ coordinates of the tokens that were used as scratchpad, and increment the position vector of only the subject token by using the first flag coordinate as an indicator on whether to change the token.
Applying these self-attention and MLP layers $k$ times will output the correct subject of the $k$-hop problem.

\item 
The embedding dimension is $D = R^k \cdot d_S + d_R + k + 1$. The transformer evaluates the sequence without Chain-of-Thought (CoT), using $k$ layers to sequentially traverse a pre-computed answer tree. The input sequence consists of $k+1$ tokens initialized from the embedding matrix as:

\begin{align*}
\text{Token } 0 \text{ (Subject):} \quad & x_0^{(0)} = \begin{bmatrix} \text{tree}^{(0)}(s_0) \\ \bm{0}_{d_R} \\ e_0 \end{bmatrix} \\
\text{Tokens } i \in [1, k] \text{ (Relations):} \quad & x_i^{(0)} = \begin{bmatrix} \bm{0}_{R^k \cdot d_S} \\ \text{bin}(r_i) \\ e_i \end{bmatrix}~.
\end{align*}

Here, $\text{tree}^{(0)}(s_0) \in \{0,1\}^{R^k \cdot d_S}$ is the concatenated binary representation of the final target subjects $s_k$ for all possible sequences of $k$ relations, ordered lexicographically. Effectively, we put the entire $k$-hop tree evaluation into the embedding of the initial subject token. The vector $e_i \in \{0,1\}^{k+1}$ is the one-hot positional encoding for the $i$-th token.

At each step $\ell \in [1, k]$, we configure a single self-attention head to fetch the tree from the preceding token, and the current relation $r_\ell$. We define the positional block as $Q_{pos} = \sum_{i=1}^k e_{i-1} e_i^T$. The attention matrices are:
\[
W_Q = \begin{bmatrix} \bm{0} & \bm{0} & \bm{0} \\ \bm{0} & \bm{0} & \bm{0} \\ \bm{0} & \bm{0} & Q_{pos} \end{bmatrix}, \quad
W_K = \begin{bmatrix} \bm{0} & \bm{0} & \bm{0} \\ \bm{0} & \bm{0} & \bm{0} \\ \bm{0} & \bm{0} & I_{k+1} \end{bmatrix}, \quad
W_V = \begin{bmatrix} I_{R^k \cdot d_S} & \bm{0} & \bm{0} \\ \bm{0} & \bm{0} & \bm{0} \\ \bm{0} & \bm{0} & \bm{0} \end{bmatrix}~.
\]

For token $\ell$, the query matches the key of token $\ell-1$. The value matrix extracts the tree subspace from token $\ell-1$ and writes it directly into the $\ell$-th token. After the attention layer, the last token contains:
\[
\tilde{x}_\ell^{(\ell)} = \begin{bmatrix} \text{tree}^{(\ell-1)}(s_\ell) \\ \text{bin}(r_\ell) \\ e_\ell \end{bmatrix}~.
\]

We now describe the MLP, which requires a width of $O(R^k \cdot d_S)$ to act as a Boolean selector over the possible sub-trees. At layer $\ell$, the active tree fetched by the attention head sits at the absolute front of the tree subspace. Its size is $V_{\ell-1} = R^{k-\ell+1} \cdot d_S$. 

Crucially, the current input tree is partitioned into exactly $R$ blocks of size $V_\ell = R^{k-\ell} \cdot d_S$. Each block corresponds to taking one of the $R$ possible relation edges from the root of the current subtree. The role of the MLP is to select one of the $R$ local branches. 

For a branch $r \in [R]$ and a bit index $j \in [V_\ell]$, the target bit is located at coordinate $c(r, j) = (r-1)V_\ell + j$ within the tree subspace. Let $e_{c(r,j)}$ be the basis vector isolating this coordinate. At layer $\ell$ we index the neurons by $(r, j)$. The weight vector and bias for each neuron are:
\[
(W)_{r,j}^{(\ell)} = \begin{bmatrix}
    e_{c(r, j)} \\ 2\text{bin}(r) - \bm{1}_{d_R} \\ 2e_\ell - \bm{1}_{k+1}
\end{bmatrix},\quad (b)_{r,j}^{(\ell)} = -\norm{\text{bin}(r)}_1 - 1~.
\]

This acts as a conditional selector for the local root. For Token $i$, if the fetched relation matches the branch, meaning $r = r_i$, then the positional and relation dot products exactly cancel the bias. The pre-activation equals the value of the target bit in the $r_i$-th subtree, which passes unchanged through the ReLU. If the relation does not match, or if the token is not at position $i$, the pre-activation is $\le -1$, outputting zero.

Finally, we use another layer of the MLP to fetch only the selected subtree out of the possible $R$, and erase the other sub-trees. The size of the new sub-tree is $V_\ell$. Applying this transformer $k$ times will output the correct solution to the $k$-hop problem.

\end{enumerate}
\end{proof}

\subsection{Proof of \Cref{thm:k-hop-cot}}\label{appen:proof_k_hop_cot}

\begin{proof}

As in the previous proof, the input to the transformer is $s,r_1,\dots,r_k$. We denote the initial subject as $s_0:= s$, and for step $i\in[k]$ we denote $s_i = g_{r_i}(s_{i-1})$. We also denote $d_S = \lceil \log_2 N \rceil$ and $d_R = \lceil \log_2 R \rceil$.

    The embedding dimension is $D = 1 + d_S + d_R + R \cdot d_S + (2k+1)$. The transformer operates autoregressively: given the prompt $s_0, r_1, \dots, r_k$, it will sequentially generate the tokens $s_1, \dots, s_k$. The sequence length over time grows to $2k+1$. The tokens are initialized from the static embedding matrix as:

\begin{align*}
\text{Token } s \text{ (Subject):} \quad & x_s^{(0)} = \begin{bmatrix} 1 \\ \text{bin}(s) \\ \bm{0}_{d_R} \\ \text{hop}(s) \\ e_{\text{pos}} \end{bmatrix} \\
\text{Token } r \text{ (Relation):} \quad & x_r^{(0)} = \begin{bmatrix} 0 \\ \bm{0}_{d_S} \\ \text{bin}(r) \\ \bm{0}_{R \cdot d_S} \\ e_{\text{pos}} \end{bmatrix}~.
\end{align*}

Here, $\text{hop}(s) = [\text{bin}(g_1(s))^T, \dots, \text{bin}(g_R(s))^T]^T \in \{0,1\}^{R \cdot d_S}$ is the concatenated binary representation of all $R$ outgoing neighbors of $s$, effectively embedding the entire 1-hop topology into the subject tokens. The vector $e_{\text{pos}} \in \{0,1\}^{2k+1}$ is a one-hot positional encoding, which is equal to $e_i$ for the $i$-th token with $i\in[2k+1]$.

At generation step $\ell \in [1, k]$, the active token responsible for predicting $s_\ell$ is the last token in the current sequence. Let this position be $p = k+\ell-1$. For $\ell=1$, the active token is $r_k$ (position $k$). For $\ell > 1$, the active token is the previously generated subject $s_{\ell-1}$ (position $k+\ell-1$).

We configure the single self-attention head to strictly route the required information from the prompt to the active generation token. Let $\tau \to \infty$ be a scaling factor. We define the position-routing block $Q_{pos} \in \mathbb{R}^{(2k+1) \times (2k+1)}$ to map the active step to its required target: $Q_{pos} = e_0 e_k^T + \sum_{\ell=1}^k e_\ell e_{k+\ell}^T$. 

The query, key, and value matrices for the single attention head are:
\[
W_Q =  \begin{bmatrix} \bm{0} & \bm{0} \\ \bm{0} & Q_{pos} \end{bmatrix}, \quad
W_K =  \begin{bmatrix} \bm{0} & \bm{0} \\ \bm{0} & I_{2k+1} \end{bmatrix}, \quad
W_V = \begin{bmatrix} 1 & \bm{0} & \bm{0} & \bm{0} & \bm{0} \\ \bm{0} & I_{d_S} & \bm{0} & \bm{0} & \bm{0} \\ \bm{0} & \bm{0} & I_{d_R} & \bm{0} & \bm{0} \\ \bm{0} & \bm{0} & \bm{0} & I_{R \cdot d_S} & \bm{0} \\ \bm{0} & \bm{0} & \bm{0} & \bm{0} & \bm{0} \end{bmatrix}~.
\]

Because the value matrix ignores positional encodings, after applying the self-attention layer and adding the residual connection we create two distinct states for the last token 
\begin{itemize}
    \item \textbf{Initialization (Position $k$):} Token $r_k$ fetches $s_0$. The output token contains $[\dots, \text{bin}(s_0)^T, \text{bin}(r_k)^T, \text{hop}(s_0)^T, e_k^T]^T$. The role is to copy the subject token to create a new token as a scratchpad.
    \item \textbf{Hop Evaluation (Positions $k+\ell$):} Token $s_{\ell-1}$ fetches $r_\ell$. The output token $[\dots, \text{bin}(s_{\ell-1})^T, \text{bin}(r_\ell)^T, \text{hop}(s_{\ell-1})^T, e_{k+\ell}^T]^T$. Notice that because $s_{\ell-1}$ is a subject token, it already contains $\text{hop}(s_{\ell-1})$.
\end{itemize}

We now describe the MLP. It is partitioned into two functional blocks. The first block activates strictly at position $k$ to copy $\text{bin}(s_0)$ to the output. This is done by implementing the identity using ReLU: $z = \text{ReLU}(z) - \text{ReLU}(-z)$ on the $d_S$ coordinates.

The second block executes the $k$-hop routing. It has a width of $O(R \cdot d_S)$, indexed by the pair $(r, j) \in [R] \times [d_S]$, corresponding to relation $r$ and the $j$-th bit of its target subject. Let $e_{(r,j)}$ be the corresponding one-hot vector, and let $E_{gen} = \sum_{\ell=1}^k e_{k+\ell}$ be the indicator for the hop-evaluation positions. The weight vector and bias are:
\[
(W)_{r,j} = \begin{bmatrix}
    0 \\ \bm{0}_{d_S} \\ 2\text{bin}(r) - \bm{1}_{d_R} \\ e_{(r,j)} \\ 2E_{gen} - \bm{1}_{2k+1}
\end{bmatrix},\quad (b)_{r,j} = -\norm{\text{bin}(r)}_1 - 1~.
\]

This acts as a Boolean selector. At positions $\ell > k$, if the fetched relation matches, i.e., $r = r_\ell$, the positional and relation dot products cancel the bias. The pre-activation becomes exactly the value of the target bit $\text{hop}(s_{\ell-1})_{(r_\ell, j)} \in \{0, 1\}$. If the relation does not match, the pre-activation is $\le -1$, outputting zero after applying the ReLU. 
The next layer of the MLP projects these $R \cdot d_S$ coordinates back to the $d_S$ subject coordinates of the correct subject. Specifically, the weight matrix encodes output $s_\ell := \text{bin}(g_{r_\ell}(s_{\ell-1}))$. The last layer zeroes out all the coordinates, except the $d_S$ subject coordinates. For the subject coordinates, it applies $\text{bin}(s)\mapsto 2\text{bin}(s) - \bm{1}_{d_S}$. 

Finally, the model will create the next CoT token by picking the token from the dictionary with the highest correlation to the last token in the sequence. Note that, by our choice, we compare only the subject coordinates representing $s_\ell$ with those of all possible tokens representing $s$. The product adds $+1$ for every bit that is $1$ in both $s_\ell$ and $s$, subtracts $1$ for every bit that is $0$ in $s_\ell$ but $1$ in $s$, and adds $0$ for every bit that is $0$ in $s$. The highest correlation is achieved for the token representing the subject $s_\ell$, which will be the next CoT token created. Applying this transformer $k$ times will output the correct answer $s_k$.
\end{proof}

\end{document}